\documentclass[12pt,a4paper,reqno]{amsart}
\usepackage{amssymb}
\usepackage{amscd}
\usepackage{enumerate}
\usepackage{graphicx}
\usepackage{siunitx}
\usepackage{tikz-cd}
\usepackage{bm}

\usepackage{url}

\usepackage{algorithm}
\usepackage{algorithmic,eqparbox,array}

\numberwithin{equation}{section}

\usepackage{mathtools}
\usepackage[tableposition=top]{caption}
\usepackage{booktabs,dcolumn}

\DeclareFontFamily{OT1}{rsfs}{}
\DeclareFontShape{OT1}{rsfs}{n}{it}{<-> rsfs10}{}
\DeclareMathAlphabet{\mathscr}{OT1}{rsfs}{n}{it}

\theoremstyle{plain}

\newtheorem{theorem}{Theorem}[section]

\theoremstyle{definition}

\newtheorem{definition}[theorem]{Definition}

\newcommand\myconv {\circ}
\newcommand{\htimes}{*}
\newcommand{\myre}{\Re}
\newcommand{\myim}{\Im}
\newcommand\mydiag[3]{\operatorname{diag}(\underset{#1}{\underbrace{#2,\cdots,#3}}) }
\newcommand\myhtrans{\mathcal{H}}
\newcommand{\mybrace}[1]{  (#1) }

\usepackage[absolute,overlay]{textpos}

\begin{document}

\title[Generalized Visual Information Analysis Via Tensorial Algebra]{Generalized Visual Information Analysis Via Tensorial Algebra}

\author{Liang Liao and Stephen John Maybank
}



\begin{textblock}{12}(5.5, 4.2)
\MakeUppercase{\footnotesize liaoliangis@126.com}, \MakeUppercase{\footnotesize sjmaybank@dcs.bbk.ac.uk}

\end{textblock}

\begin{abstract}  
\vspace{2em}
Higher order data is modeled using matrices whose entries are numerical arrays of a fixed size. These arrays, called t-scalars, form a commutative ring under the convolution product. Matrices with elements in the ring of t-scalars are referred to as t-matrices. The t-matrices can be scaled, added and multiplied in the usual way. There are t-matrix generalizations of positive matrices, orthogonal matrices and Hermitian symmetric matrices. With the t-matrix model, it is possible to generalize many well-known matrix algorithms. In particular, the t-matrices are used to generalize the SVD (Singular Value Decomposition), HOSVD (High Order SVD), PCA (Principal Component Analysis), 2DPCA (Two Dimensional PCA) and GCA (Grassmannian Component Analysis). The generalized t-matrix algorithms, namely TSVD, THOSVD, TPCA, T2DPCA and TGCA, are applied to low-rank approximation, reconstruction, and supervised classification of images. Experiments show that the t-matrix algorithms compare favorably with standard matrix algorithms.

~\\
\textsc{keywords}.~~
Commutative ring, Generalized scalars, 
{Grassmannian manifold}, Image Analysis,
Tensor singular value decomposition,
Tensors
\end{abstract}

\maketitle



\section{Introduction}

In data analysis, machine learning and computer vision, the data are often given in the form of multi-dimensional arrays of numbers. For example, an RGB image has three dimensions, namely two for the pixel array and a third dimension for the values of the pixels. An RGB image is said to be an array of order three. Alternatively, the RGB image is said to have three modes or to be three-way. A video sequence of images is of order four, with two dimensions for the pixel array, one dimension for time and a fourth dimension for the pixel values.

One way of analyzing multi-dimensional data is to remove the array structure by flattening, to obtain a vector. A set of vectors obtained in this way can be analyzed using standard matrix-vector algorithms such as the singular value decomposition (SVD) and principal components analysis (PCA). An alternative to flattening is to use algorithms that preserve the multi-dimensional structure. In these algorithms, the elements of matrices and vectors are entire arrays rather than real numbers in $\mathbb{R}$ or complex numbers in $\mathbb{C}$. Multi-dimensional arrays with the same dimensions can be added in the usual way, but there is no definition of multiplication which satisfies the requirements for a field such as $\mathbb{R}$ or $\mathbb{C}$. However, multiplication based on the convolution product has many but not all of the properties of a field. Convolution multiplication differs from the multiplication in a field in that many elements have no multiplicative inverse. The multi-dimensional arrays with given dimensions form a commutative ring under the convolution product. The elements of this ring are referred to as t-scalars.

An application of the Fourier transform shows that each ring of t-scalars under the convolution product is isomorphic to a ring of arrays in which the Hadamard product defines the multiplication. In effect, the ring obtained by applying the Fourier transform splits into a product of copies of $\mathbb{C}$. It is this splitting which allows the construction of new algorithms for analyzing tensorial data without flattening. The so-called t-matrices with t-scalar entries have many of the properties of matrices with elements in $\mathbb{R}$ or $\mathbb{C}$. In particular, t-matrices can be scaled, added and multiplied. There is an additive identity and a multiplicative identity. The determinant of a t-matrix is defined and a given t-matrix is invertible if and only if it has an invertible determinant. The t-matrices include generalizations of positive matrices, orthogonal matrices and symmetric matrices.

A tensorial version, TSVD, of the SVD is described in \cite{Kilmer2011Factorization-TProduct001-0006} and \cite{liaoliang-00016}. The TSVD expresses a t-matrix as the product of three t-matrices, of which two are generalizations of the orthogonal matrices and one is a diagonal matrix with positive t-scalars on the diagonal. The TSVD is used to define tensorial versions of principal components analysis (PCA) and two dimensional PCA (2DPCA). A tensorial version of Grassmannian components analysis is also defined. These tensorial algorithms are tested by experiments that include low-rank approximations to tensors, reconstruction of tensors and terrain classification using hyperspectral images. The different algorithms are compared using the peak signal to noise ratio and Cohen's kappa.

The t-scalars are described in Section \ref{section:tensor-algebra} and the t-matrices are described in Section \ref{section:generalized-matrix}. The TSVD is described in Section \ref{section:generalized-SVD}. A tensorial version of principal components analysis (TPCA) is obtained from the TSVD in Section \ref{section:image-analysis-SVD} and then generalized to tensorial two dimensional PCA (T2DPCA). A tensorial version of Grassmannian components analysis is also defined. The tensorial algorithms are tested experimentally in Section \ref{section:experiments}. Some concluding remarks are made in section \ref{section:conclusions}.

\subsection{Related Work}

A tensor of order two or more can be simplified using the so-called
$N$ mode singular value decomposition (SVD). The three mode case is described by Tucker in \cite{tucker1966some-00012}. The multi-modal case is discussed in detail by De Lathauwer et al. in \cite{de2000multilinear-0007}. Each mode of the tensor has an associated set of vectors, each one of which is obtained by varying the index for the given mode while keeping the indices of the other modes fixed. In the $N$ mode SVD, an orthonormal basis is obtained for the space spanned by these vectors. In the 2-mode case, the result is the usual SVD.  The resulting decomposition of a tensor is referred to as the higher-order SVD (HOSVD). Surveys of tensor decompositions can be found in Kolda and Bader \cite{Kolda2009TensorDecompositions} and Sidiropoulos et al. \cite{Sidiropoulos2017TensorDecomposition}. De Lathauwer et al. \cite{de2000multilinear-0007} describe a higher-order eigenvalue decomposition. Vasilescu and Terzopoulos \cite{vasilescu2002multilinear-00014} use the $N$ mode SVD to simplify a fifth-order tensor constructed from face images taken under varying conditions and with varying expressions. A tensor version of the singular value decomposition is described in \cite{Kilmer2011Factorization-TProduct001-0006}, \cite{liaoliang-00016}, and \cite{Kilmer2013Third-TProduct003-0005}.

He et al. \cite{He2016Tensor-0008} sample a hyperspectral data cube to yield tensors of order three of which two orders are for the pixel array and one order is for the hyperspectral bands. A training set of samples is used to produce a dictionary for sparse classification. Lu et al. \cite{lu2008mpca-00010} use $N$-mode analysis to obtain projections of tensors to a lower-dimensional space. The resulting multilinear PCA is applied to the classification of gait images. Vannieuwenhoven et al. \cite{vannieuwenhoven2012new-00013} describe a new method for truncating the higher-order SVD, to obtain low-rank multilinear approximations to tensors. The method is tested on the classification of handwritten  digits and the compression of a database of face images.

Many authors have studied algebras of matrices in which the elements are tensors of order one, equipped with a convolution multiplication, under which they form a commutative ring $R$ with a multiplicative identity. In particular, Gleich et al. \cite{Gleich2014The-0004} describe the generalized eigenvalues and eigenvectors of matrices with elements in $R$ and show how the standard power method for finding an eigenvector and the standard Arnoldi method for constructing an orthogonal basis for a Krylov subspace can both be generalized.  Braman \cite{Braman2010Third-TProduct002-0001} shows that the t-vectors with a given dimension form a free module over $R$. Kilmer and Martin \cite{Kilmer2011Factorization-TProduct001-0006} show that many of the properties and structures of canonical matrices and vectors can be generalized. Their examples include transposition, orthogonality and the singular value decomposition (SVD). The tensor SVD is used to compress tensors. A tensor-based method for image de-blurring is also described. Kilmer et al. \cite{Kilmer2013Third-TProduct003-0005} generalize the inner product of two vectors, suggest a notion of
the angle between two vectors with elements in $R$, and define a notion of orthogonality for two vectors. A generalization of the Gram-Schmidt method for generating an orthonormal set of vectors is also described in \cite{Kilmer2013Third-TProduct003-0005}.

Zhang et al. \cite{liaoliang-00016} use the tensor SVD to store video sequences efficiently and also to fill in missing entries in video sequences. Zhang et al. \cite{Zhang2016A-00015} use a randomized  version of the tensor SVD to produce low-rank approximations to matrices. Ren et al. \cite{Ren2017Hyperspectral-00011} define a tensor version of principal component analysis and use it to extract features from hyperspectral images. The features are classified using standard methods such as support vector machines and nearest neighbors.
Liao et al. \cite{Liao2017Supervised-0009} generalize a sparse representation classifier to tensor data and apply the generalized classifier to image data such as numerals and faces. Chen et al. \cite{Chen2015Change-0002} use a four-dimensional HOSVD to detect changes in a time sequence of hyperspectral images. The K-means clustering algorithm is used to classify the pixel values as changed or unchanged. Fan et al. \cite{Fan2018Spatial-0003} model a hyperspectral image as the sum of an ideal image, a sparse noise term and a Gaussian noise term.
A product of two low-rank tensors models the ideal image.
The low-rank tensors are estimated by minimizing  a penalty function obtained by adding the squared errors in a fit of the hyperspectral image to penalty terms for the sparse noise and the sizes of the two low-rank tensors. Lu et al. \cite{Lu2016TensorRobust} approximate a third-order tensor using the sum of a low-rank tensor and a sparse tensor. Under suitable conditions, the low-rank tensor and the sparse tensor are recovered exactly.

\section{T-scalars}
\label{section:tensor-algebra}

The notations  for t-scalars are summarized in Section \ref{section:notation}. Basic definitions are given in Section \ref{section:Definitions}. The Fourier transform of a t-scalar is defined in Section \ref{FourierTransformOfAT-Scalar}. Properties of t-scalars and the Fourier transform of a t-scalar are described in Section \ref{section:PropertiesOft-scalars}. A generalization of the t-scalars is described in Section \ref{subsection:Generalizedt-scalars}.

\subsection{Notations and Preliminaries}
\label{section:notation}

An array of order $N$ over the complex numbers $\mathbb{C}$ is an element of the set $C$ defined by $C\equiv\mathbb{C}^{I_{1}\times \ldots \times I_{N}}$, where the
$I_{n}$ for $1\le n \le N$ are strictly positive integers. Similarly, an array of order $N$ over the real numbers is an element of the set $R$ defined by $R\equiv\mathbb{R}^{I_{1}\times\ldots\times I_{N}}$.
The sets $R$ and $C$ have the structure of commutative rings, in which the product is defined by circular convolution. The elements of $C$ and $R$ are referred to as t-scalars.

Elements of $\mathbb{R}$ and $\mathbb{C}$ are denoted by lower case letters and
tensorial data are denoted by upper case letters. The t-scalars are identified using the subscript $T$, for example, $X_{T}$. Lower case subscripts such as $i$, $j$, $\alpha$, $\beta$ are indices or lists of indices.

All indices begin from $1$ rather than $0$. Given an array of any order $N$, namely
$X \in \mathbb{C}^{I_1\times I_2 \times \cdots \times I_N}$  ($N \geqslant 1 $),
$X_{i_1, i_2, \cdots, i_N}$
or $(X)_{i_1, i_2, \cdots, i_N}$
denote its $(i_1, i_2, \cdots,  i_N)$-th entry in $\mathbb{C}$. The notation $X_{i}$, or $(X)_{i}$, is also used, where $i$ is a multi-index defined by $i = (i_{1},\cdots, i_{N})$. Let $I=(I_{1},I_{2}, \ldots, I_{N})$ and let $i$ be a multi-index. The notation $1\le i\le I$ specifies the range of values of $i$ such that
$1\le i_{n}\le I_{n}$ for $1\le n\le N$. It is often convenient to extend the indexing beyond the range specified by $I$. Let $j$ be a general multi-index. Then $X_{j}$ is defined by $X_{j} = X_{i}$, where $i$ is the multi-index such that each component
$i_{n}$ is in the range $1\le i_{n}\le I_{n}$ and $i_{n}-j_{n}$ is divisible by $I_{n}$. A multi-index such as $i-j+1$ has components $i_{n}-j_{n}+1$ for $1\le n\le N$. The sum $\sum\nolimits_{i=1}^{I} (\cdot)$ is an abbreviation for
$
\sum\nolimits_{i_{1}=1}^{I_{1}}\ldots \sum_{i_{N}=1}^{I_{N}} (\cdot).
$

\subsection{Definitions}
\label{section:Definitions}

The following definitions are for t-scalars in $C$. Similar definitions can be made for t-scalars in $R$.

\begin{definition}
\label{definition:tensor-addition}
T-scalar addition.~
Given t-scalars $X_\mathit{T}$ and $Y_\mathit{T} $ in  $C$, the addition of $X_\mathit{T}$ and $Y_\mathit{T} $
denoted by $D_\mathit{T} \doteq X_\mathit{T} + Y_\mathit{T} $
is element-wise,
\begin{equation}
D_{T,i} = X_{T,i}+Y_{T,i},~~1\le i\le I.
\end{equation}
\end{definition}

\begin{definition}
\label{definition:tensor-product}
{T-scalar multiplication}.
Given t-scalars
$X_\mathit{T}$ and $Y_\mathit{T}$ in  $C$,
their product, denoted by $D_\mathit{T}= X_\mathit{T} \myconv Y_\mathit{T} $
is a t-scalar in $C$ defined by the circular convolution
\begin{equation}
\label{equation:002}
D_{T,i} = \sum_{j=1}^{I}X_{T,i-j+1}Y_{T,j},~~1\le i\le I.
\end{equation}
\end{definition}

Definitions \ref{definition:tensor-addition} and \ref{definition:tensor-product} reduce to complex number addition and multiplication when $N = 1$ and $I_{1} = 1$.

\begin{definition}
\label{definition:tensor-zero}
{Zero t-scalar}.~ The zero t-scalar
$Z_\mathit{T}$ is the array in $C$ defined by
\begin{equation}
Z_{T,i} = 0,~~1\le i\le I.
\end{equation}
\end{definition}
For all t-scalars $X_\mathit{T}$,
$X_\mathit{T} + Z_\mathit{T} = X_\mathit{T} $ and $X_\mathit{T} \myconv Z_\mathit{T} = Z_\mathit{T}$.

\begin{definition}
\label{definition:tensor-identity}
{Identity t-scalar.}~
The identity t-scalar
$E_\mathit{T}$
in $C$ has the first entry equal to $1$ and all other entries equal to $0$, namely,
$E_{T,i} = 1$ if $i = (1, \cdots, 1)$ and $E_{T,i} = 0$ otherwise.
\end{definition}
For all t-scalars $X_\mathit{T} \in C$,  $X_\mathit{T} \myconv E_\mathit{T} \equiv X_\mathit{T}$.

The set of t-scalars satisfies the axioms of a commutative ring with $Z_{T}$ as an additive identity and $E_{T}$ as a multiplicative identity. This ring of t-scalars is denoted by $(C, +, \myconv)$.
The ring $(C, +, \myconv)$ is a generalization of the field $(\mathbb{C}, +, \cdot)$ of complex numbers. If the t-scalars are restricted to have real number elements, then the ring $(R, +, \myconv)$ is obtained.

\subsection{Fourier Transform of a T-scalar}
\label{FourierTransformOfAT-Scalar}


Let $\zeta_{n}$ be a primitive $I_{n}$-th root of unity, for example,
\[
\zeta_{n} = \exp\left(2\pi\sqrt{-1}/I_{n}\right),~~1\le n\le N.
\]
Let $\overline{\zeta}_{n}$ be the complex conjugate of $\zeta_{n}$ and let $X_{T}$ be
a t-scalar in the ring $C$. The Fourier transform $F(X_{T})$ of $X_{T}$ is defined by
\[
F(X_{T})_{i} = \sum\limits_{j=1}^{I}X_{T,j}\cdot \zeta_{1}^{(i_{1}-1)(j_{1}-1)}\cdots \zeta_{N}^{(i_{N}-1)(j_{N}-1)}
\]
for all indices $1\le i\le I$.

The inverse of the Fourier transform is defined by
\[
\begin{aligned}
X_{T, i} = \frac{\sum\nolimits_{j=1}^{I}F(X_{T})_{j}\cdot \overline{\zeta}_{1}^{(i_{1}-1)(j_{1}-1)}\cdots
\overline{\zeta}_{N}^{(i_{N}-1)(j_{N}-1)}}{I_{1}\cdots I_{N}}
\end{aligned}
\]
for all indices $1\le i\le I$.

{
Given t-scalars $X_T \in C$ and $Y_{T} \in C$ and their t-scalar product $D_{T} = X_T \myconv Y_T $, it follows that
\begin{equation}
\label{equation:006}
F(D_\mathit{T}) = F(X_\mathit{T}) \htimes F(Y_\mathit{T}),
\end{equation}
where $\htimes$ denotes the Hadamard product in $C$. Equation (\ref{equation:006}) is an extension of the convolution theorem \cite{bracewell1986fourier}. The equation can be equivalently rewritten as
\begin{equation}
\label{equation:Hadamard-product}
F(D_\mathit{T})_{i} = F(X_\mathit{T})_{i} \cdot F(Y_\mathit{T})_{i},~~1\le i\le I,
\end{equation}
where $\cdot$ is multiplication in $\mathbb{C}$.

An equivalent definition of the Fourier transform of a higher-order array in the form of multi-mode tensor multiplication and a diagram of the multiplication of two t-scalars, computed in the Fourier domain, is given in a supplementary file.
}

It is not difficult to prove that $C$ is a commutative ring, $(C,+, \htimes)$, under the Hadamard product. The Fourier transform is a ring isomorphism from $(C, +, \myconv)$ to $(C, +, \htimes)$. The identity element of $(C, +, \htimes)$ is $J_{T} = F(E_{T})$. All the entries of $J_{T}$ are equal to 1.

\subsection{Properties of t-scalars}
\label{section:PropertiesOft-scalars}

The invertible t-scalars are defined as follows.

\begin{definition}
\label{definition-invertiable-t-scalar}
{Invertible t-scalar}:
Given a t-scalar $X_\mathit{T}$, if there exists a t-scalar $Y_\mathit{T}$ satisfying $X_\mathit{T} \myconv Y_\mathit{T} = E_\mathit{T}$, then $X_\mathit{T}$ is said to be invertible. The t-scalar $Y_\mathit{T}$ is the inverse
of $X_\mathit{T}$ and denoted by $Y_\mathit{T} \doteq X_\mathit{T}^{-1}
\doteq E_\mathit{T} / X_\mathit{T}\;.
$
\end{definition}

The zero t-scalar $Z_\mathit{T}$ is non-invertible. In addition, there is
an infinite number of t-scalars that are non-invertible. For example, given a t-scalar $X_{T} \in C$, if the entries of $X_{T}$ are all equal, then $X_{T}$ is non-invertible. The existence of more than one non-invertible element shows that $C$ is not a field.

\begin{definition}
\label{scalar-multiplication}
{Scalar multiplication of a t-scalar}.~
Given a scalar $\lambda \in \mathbb{C}$ and
a t-scalar
$X_\mathit{T} \in C$, their product,
denoted by
$Y_\mathit{T} = \lambda \cdot X_\mathit{T} \equiv
X_\mathit{T} \cdot \lambda$
is the t-scalar given by
\begin{equation}
Y_\mathit{T,i} = \lambda \cdot X_\mathit{T,i},~~1\le i\le I.
\end{equation}
\end{definition}
It can be shown that the set of t-scalars is a vector space over $\mathbb{C}$.

The following definition of the conjugate of a t-scalar generalizes the conjugate of a complex number.
\begin{definition}
\label{definition-transition-of-a-tensor}
Conjugate of a t-scalar.~Given a t-scalar $X_\mathit{T}$ in $C$,
its conjugate, denoted by $\operatorname{conj}(X_\mathit{T}) $,
is the t-scalar in $C$ such that
\begin{equation}
\operatorname{conj}(X_\mathit{T})_{i} = \overline{X_{T, 2-i}}~,~~1\le i\le I,
\label{equation:entry-of-conjugate-t-scalar}
\end{equation}
where $\overline{X_{T, 2-i}}$ is the complex conjugate of $X_{T, 2-i}$ in $\mathbb{C}$.
\end{definition}

The conjugate of a t-scalar reduces to the conjugate of a complex number when $N = 1$, $I_{1} = 1$.
The relationship of $\operatorname{conj}(X_\mathit{T})$ and
$X_\mathit{T}$ is much clearer if they are mapped to the Fourier domain -- each entry of $F(\operatorname{conj}(X_\mathit{T}))$
is the complex conjugate of the corresponding entry of
$F(X_\mathit{T})$, namely
\begin{equation}
\label{equation:009}
F(\operatorname{conj}(X_\mathit{T}))_{i} =
\overline{F(X_\mathit{T})_{i}}~,~~1\le i\le I.
\end{equation}
It follows from equation (\ref{equation:entry-of-conjugate-t-scalar}) that $\operatorname{conj}(\operatorname{conj}(X_\mathit{T}) ) = X_\mathit{T} $ for any $X_\mathit{T} \in C$.

\begin{definition}
\label{definition-real-t-scalar}
Self-conjugate t-scalar:
Given a t-scalar $X_\mathit{T} \in C$, if $X_\mathit{T} = \operatorname{conj}(X_\mathit{T}) $, then $X_\mathit{T}$ is said to be a self-conjugate t-scalar.
\end{definition}

If $X_{T}$ is self conjugate, then
\begin{equation}
\overline{F(X_{T})_{i}} = F(\operatorname{conj}(X_{T}))_{i} = F(X_{T})_{i} \in \mathbb{C}~,1 \le i \le I.
\label{Fourierconj}
\end{equation}
It follows
from equation (\ref{Fourierconj})
that $X_{T}$ is self-conjugate if and only if all the elements of $F(X_{T})$ are real numbers.

The t-scalars $Z_\mathit{T}$
and $E_\mathit{T}$ are both self-conjugate. Furthermore,
the self-conjugate t-scalars form a ring denoted by $C^\mathit{sc}$. This ring is a subring of $C$.

Given any t-scalar ${X}_\mathit{T} \in C$, let $\myre({X}_\mathit{T})$ and $\myim({X}_\mathit{T})$ be defined by
\begin{eqnarray}
\myre(X_{\mathit{T}}) &=& 2^{-1}(X_{\mathit{T}}+\operatorname{conj}(X_{\mathit{T}})), \label{A1}\\
\myim(X_{\mathit{T}})&=& \left(2\sqrt{-1}\right)^{-1}(X_{\mathit{T}}-\operatorname{conj}(X_{\mathit{T}})). \label{A2}
\end{eqnarray}
It follows from equation (\ref{Fourierconj}) that $\myre(X_{\mathit{T}})$ and $\myim(X_{\mathit{T}})$ are self-conjugate. The t-scalars
$X_\mathit{T} \in C$ and $\operatorname{conj}(X_\mathit{T}) \in C$ can be expressed in the form
\begin{eqnarray}
X_\mathit{T} &=&  \myre(X_{\mathit{T}}) + \sqrt{-1}\myim(X_{\mathit{T}}),\label{A3}\\
\operatorname{conj}(X_\mathit{T}) &=&  \myre(X_{\mathit{T}}) - \sqrt{-1}\myim(X_{\mathit{T}}).\label{A4}
\end{eqnarray}
In an analogy with the real and imaginary parts of a complex number, $\myre(X_\mathit{T})$ is called
the real part of $X_{T}$ and $\myim(X_{\mathit{T}})$ is called the imaginary part of $X_{T}$.

Given two t-scalars $X_\mathit{T}$ and $Y_\mathit{T}$, the equations (\ref{equation:backward-compatible}) hold true and are backward compatible with the corresponding equations for complex numbers.
\begin{figure*}[htb]
\begin{equation}
\begin{matrix}
X_T + Y_T \equiv
\Big(\myre{(X_\mathit{T})} + \myre{(Y_\mathit{T})} \Big) +
\sqrt{-1}\cdot \Big(\myim{(X_\mathit{T})} + \myim{(Y_\mathit{T})} \Big)  \\
X_\mathit{T} \myconv Y_\mathit{T} \equiv
\Big(
\myre{(X_\mathit{T})} \myconv \myre{(Y_\mathit{T}) } -
\myim{(X_\mathit{T})} \myconv \myim{(Y_\mathit{T}) }
\Big)
+
\sqrt{-1} \Big(
\myim{(X_\mathit{T})} \myconv \myre{(Y_\mathit{T}) } +
\myre{(X_\mathit{T})} \myconv \myim{(Y_\mathit{T}) }
\Big) \\
\operatorname{conj}(X_\mathit{T})  \myconv X_\mathit{T}
\equiv
X_\mathit{T} \myconv \operatorname{conj}(X_\mathit{T}) \equiv
\myre(X_\mathit{T})^{2} +
\myim(X_\mathit{T})^{2}
\end{matrix}
\label{equation:backward-compatible}
\end{equation}
\end{figure*}

\begin{definition}
\label{definition-nonnegative-t-scalar}
{Nonnegative t-scalar}:
The t-scalar $X_\mathit{T}$ is said to be nonnegative if there exists
a self-conjugate t-scalar $Y_\mathit{T}$ such that $X_\mathit{T} = Y_\mathit{T} \myconv Y_\mathit{T} \doteq Y_\mathit{T}^{2}$.
\end{definition}

If a t-scalar $X_{T}$ is nonnegative, it is also self-conjugate, because the multiplication of any two self-conjugate t-scalars is also a self-conjugate t-scalar. Thus, both $Z_\mathit{T}$ and $E_\mathit{T}$ are nonnegative, since $Z_{T}$ and $E_{T}$ are self-conjugate t-scalars and satisfy
$Z_\mathit{T} = Z_\mathit{T}^{2}$ and $E_\mathit{T} = E_\mathit{T}^{2}$. Furthermore, for all $X_{T} \in C$, the ring element
$\myre(X_{T})^{2} + \myim(X_{T})^{2}$ is nonnegative.

The set $S^\mathit{nonneg}$ of nonnegative t-scalars is closed under the t-scalar addition and multiplication. Since a nonnegative t-scalar is also a self-conjugate t-scalar,
$S^\mathit{nonneg} \subset C^\mathit{sc} \subset C$.

\begin{theorem}
For all t-scalars $X_\mathit{T} \in S^\mathit{nonneg}$, there exists a unique
t-scalar $S_\mathit{T} \in S^\mathit{nonneg}$ satisfying $X_\mathit{T} = S_\mathit{T} \myconv S_\mathit{T} \doteq S_\mathit{T}^{2}$.
We call the nonnegative t-scalar $S_\mathit{T}$ the arithmetic square root of the nonnegative t-scalar $X_\mathit{T}$ and denote it by
\begin{equation}
S_\mathit{T} \doteq \sqrt{X_\mathit{T}} \doteq  X_\mathit{T}^{{1}/{2}}
\end{equation}
\end{theorem}

\begin{proof}
Let $X_{\mathit{T}}=Y_{\mathit{T}}\myconv Y_{\mathit{T}}$, such that $Y_{\mathit{T}}$ is self-conjugate. On applying the Fourier transform, it follows that
\[
F(X_{\mathit{T}})_{i} = F(Y_{T})_{i}^{2} \ge 0,~~1\le i\le I.
\]
Let $S_{T}$ be defined such that
\[
F(S_{\mathit{T}})_{i} = (F(X_{\mathit{T}})_{i})^{1/2},~~1\le i\le I,
\]
where the nonnegative square root is chosen for each value of $i$. The Fourier components $F(S_{\mathit{T}})_{i}$ are real valued, thus $S_{\mathit{T}}$ is self conjugate. The equation $X_{\mathit{T}} = S_{\mathit{T}} \myconv S_{\mathit{T}}$ holds because the Fourier transform is injective.
\end{proof}

\begin{definition}
\label{definition:positve-t-scalar}
A nonnegative t-scalar that is invertible under  multiplication is called a positive t-scalar. The set of positive t-scalars is denoted by $S^\mathit{pos}$.
\end{definition}

The following inclusions are strict, $S^\mathit{pos} \subset S^\mathit{nonneg} \subset C^\mathit{sc} \subset C$. The inverse and the arithmetic square root of a positive t-scalar are positive.

The absolute t-value $r(X_{\mathit{T}})$ of $X_{\mathit{T}}$ is defined by
\begin{equation}
r(X_{\mathit{T}}) = \sqrt{\myre(X_{\mathit{T}})^{2}+\myim(X_{\mathit{T}})^{2}}.
\label{A5}
\end{equation}
The t-scalars $\myre(X_{\mathit{T}})$ and $\myim(X_{\mathit{T}})$ are both self-conjugate, therefore $\myre(X_{\mathit{T}})^{2}$ and $\myim(X_{\mathit{T}})^{2}$ are both {nonnegative} . The sum $\myre(X_{\mathit{T}})^{2}+\myim(X_{\mathit{T}})^{2}$ is {nonnegative} and it has a
{nonnegative} arithmetical square root, namely $r(X_T)$.

If $r(X_{\mathit{T}})$ is invertible, then let $\phi(X_{T})$ be defined by
\begin{equation}
\phi(X_{\mathit{T}}) \doteq r(X_{\mathit{T}})^{-1}\myconv X_T.
\label{A6}
\end{equation}
The ring element $\phi(X_{\mathit{T}})$ is a generalized angle. The order $1$ version of $\phi(X_{\mathit{T}})$ is obtained by Gleich et al. in \cite{Gleich2014The-0004}. Equation (\ref{A6}) generalizes the polar form of a complex number. It can be shown that
\[
\phi(X_{\mathit{T}})\myconv\operatorname{conj}(\phi(X_T)) = E_{T}.
\]
The absolute t-value $r(X_{\mathit{T}})$ is used in Section \ref{section:generalized-matrix} to define a generalization of the Frobenius norm for t-matrices.

\section{Matrices with T-Scalar Elements}
\label{section:generalized-matrix}

It is shown that t-matrices, i.e. matrices with elements in the rings $C$ or $R$,
are in many ways analogous to matrices with elements in $\mathbb{C}$ or $\mathbb{R}$.

\subsection{Indexing}
\label{section:ConceptsAndIndexingNotations}

The t-matrices are
order-two arrays of t-scalars. Since the t-scalars are arrays of complex numbers, it is convenient to organize t-matrices as hierarchical arrays of complex numbers.

Let $X_\mathit{TM}$ be a t-matrix with $D_1$ rows and $D_2$ columns.
Then $X_\mathit{TM}$ is an element of $C^{D_1\times D_2}$. The $(\alpha, \beta)$ entry of $X_\mathit{TM}$ is the element of $C$ denoted by $X_{\mathit{TM},\alpha,\beta}$ for $1\le\alpha\le D_1$ and
$1\le\beta\le D_2$. Let $i$ be a multi-index for elements of $C$. Then
$X_{\mathit{TM}, i, \alpha,\beta}$ is the element of $\mathbb{C}$ given
as the $i$-th entry of the ring element $X_{\mathit{TM},\alpha,\beta}$.

The t-matrix $X_\mathit{TM}$ can be interpreted as an element in $\mathbb{C}^{I_1\times \cdots \times I_N\times D_1\times D_2}$ or alternatively it can be interpreted as an element in $\mathbb{C}^{D_1\times D_2\times I_1\times \cdots \times I_N}$. The only thing needed to switch from one data structure to the other is a permutation of indices. The data structure $\mathbb{C}^{I_1\times \cdots \times I_N\times D_1\times D_2}$ is chosen unless otherwise indicated.

\subsection{Properties of t-matrices}

(1) {\bf T-matrix addition}:
Given any t-matrices
${A}_\mathit{TM} \in C^{D_1 \times D_2}$
and
${B}_\mathit{TM} \in C^{D_1\times D_2}$, the addition,
denoted by
${C}_\mathit{TM} \doteq {A}_\mathit{TM} + {B}_\mathit{TM} \in C^{D_1\times D_2}$, is entry-wise, such that
$
C_\mathit{TM,\alpha,\beta} = A_\mathit{TM,\alpha,\beta} + B_\mathit{TM,\alpha,\beta}$, for
$1\le \alpha\le D_1$ and $1\le\beta\le D_2$.

(2) {\bf T-matrix multiplication}:
Given any t-matrices ${A}_\mathit{TM} \in C^{D_1\times Q}$  and ${B}_\mathit{TM} \in C^{Q\times D_2}$,
their product, denoted by  ${C}_\mathit{TM} \doteq {A}_\mathit{TM} \myconv {B}_\mathit{TM} $, is the t-matrix in
$C^{D_1\times D_2}$ defined by
\[
\begin{matrix}
C_\mathit{TM,\alpha,\beta}
= \sum\nolimits_{\gamma=1}^{Q}
A_\mathit{TM,\alpha,\gamma} \myconv B_\mathit{TM,\gamma, \beta}
\end{matrix}
\]
for all indices
$1\le\alpha\le D_1, 1\le \beta\le D_2$.

An example of t-matrix multiplication {$C_\mathit{TM} = A_\mathit{TM} \circ B_\mathit{TM}$ {$\in$} } $C^{2\times 1} \equiv \mathbb{C}^{3\times 3\times 2\times 1}$ where $A_\mathit{TM} $$\in$ $C^{2\times 2} \equiv
\mathbb{C}^{3\times 3\times 2\times 2}$ and
$B_\mathit{TM} $$\in$ $C^{2\times 1} \equiv
\mathbb{C}^{3\times 3\times 2\times 1}$ is given in
{a supplementary file}.

(3) {\bf Identity t-matrix} : The identity t-matrix is the diagonal t-matrix, in which each diagonal entry is equal to the identity t-scalar $E_\mathit{T}$ in Definition
\ref{definition:tensor-identity}. The $D\times D$ identity t-matrix is denoted by
$I_\mathit{TM}^{(D)} \doteq \mydiag{D}{E_{T}}{E_{T}} $.

Given any $X_\mathit{TM} \in C^{D_1\times D_2}$, it follows that $I_\mathit{TM}^{(D_1)} \myconv X_\mathit{TM} =  X_\mathit{TM} \myconv I_\mathit{TM}^{(D_2)}  = X_\mathit{TM}$. The identity t-matrix $I_\mathit{TM}^{(D)}$ is also denoted by $I_\mathit{TM}$ if the value of $D$ can be inferred from context.

(4) {\bf Scalar multiplication}:
Given any ${A}_\mathit{TM} \in C^{D_1\times D_2}$  and $\lambda \in \mathbb{C}$,
their multiplication, denoted by  ${B}_\mathit{TM} \doteq
\lambda \cdot
{A}_\mathit{TM} $, is the t-matrix in $C^{D_1\times D_2}$ defined by
\[
B_{\mathit{TM},\alpha,\beta} = \lambda \cdot
A_{\mathit{TM},\alpha,\beta},~~1\le \alpha\le D_1, 1\le \beta\le D_2.
\]
where the products with $\lambda$ are computed as in Definition \ref{scalar-multiplication}.

(5) {\bf T-scalar multiplication}:
Given any ${A}_\mathit{TM} \in C^{D_1\times D_2}$  and $\lambda_{T} \in C$,
their product, denoted by  ${B}_\mathit{TM} \doteq
\lambda_{T} \myconv {A}_\mathit{TM} $, is the t-matrix in $C^{D_1\times D_2}$ defined by
\[
B_{\mathit{TM},\alpha,\beta} = \lambda_{T} \myconv
A_{\mathit{TM},\alpha,\beta},~~1\le\alpha\le D_1, 1\le \beta\le D_2.
\]

(6) {\bf Conjugate transpose of a t-matrix} : {Given any} t-matrix $X_\mathit{TM} \in C^{D_1\times D_2}$, its conjugate transpose, denoted by $X_\mathit{TM}^{\myhtrans}$ is the t-matrix in $C^{D_2\times D_1}$ given by
\[
X_{\mathit{TM},\beta,\alpha}^{\myhtrans} = \operatorname{conj}(
X_{\mathit{TM},\alpha,\beta})\in C,~~1\le \alpha\le D_1, 1\le \beta\le D_2.
\]
A square matrix $U_\mathit{TM}$ is said to orthogonal if $U_\mathit{TM}^{\myhtrans}$ is the inverse t-matrix of
$U_\mathit{TM}$, i.e., $U_\mathit{TM}^{\myhtrans} \circ U_\mathit{TM} = U_\mathit{TM} \circ U_\mathit{TM}^{\myhtrans} = I_\mathit{TM}$.
The Fourier transform $F$ is extended to t-matrices element-wise, i.e. {$F(X_\mathit{TM})$} is the $D_1\times D_2$ t-matrix defined by
\begin{equation}
\label{equation:fourier-transform-t-matrix}
F(X_\mathit{TM})_{\alpha, \beta} = F(X_{\mathit{TM}, \alpha, \beta})\;\;.
\end{equation}
for all indices
$1 \le \alpha \le D_1$ and
$1 \le \beta \le D_2$.

It is not difficult to prove that
\[
F(X_\mathit{TM}^{\myhtrans})_{i,\beta,\alpha} =
\overline{F(X_\mathit{TM})_{i, \alpha,\beta}} \in \mathbb{C}~
\]
for all indices
$1\le i\le I, 1\le \alpha\le D_1, 1\le \beta\le D_2$.

(7) {\bf T-vector dot product and the Frobenius norm}:
Given any two t-vectors (i.e., two t-matrices, each having only one column) $X_\mathit{TV} $ and $Y_\mathit{TV} $ of the same length $D$, their dot product is the t-scalar defined by
\[
\langle X_\mathit{TV},  Y_\mathit{TV} \rangle \doteq \sum_{\alpha=1}^{D} \operatorname{conj}(X_\mathit{TV,\alpha}) \myconv
Y_{\mathit{TV},\alpha}\;\;.
\]
If $\langle X_\mathit{TV},  Y_\mathit{TV} \rangle = Z_{T}$, then $X_\mathit{TV}$ and
$Y_\mathit{TV}$ are said to be orthogonal. The nonnegative t-scalar
$\sqrt{\langle X_\mathit{TV},  X_\mathit{TV} \rangle}$
is called the generalized norm of $X_\mathit{TV}$ and denoted by
\begin{equation}
\|X_\mathit{TV}\|_{F} \doteq \sqrt{\langle X_\mathit{TV},  X_\mathit{TV} \rangle} \equiv \left({\sum\limits_{\alpha=1}^{D}r(X_{\mathit{TV},\alpha})^{2}}\right)^{1/2}
\label{equation:generalized-t-scalar-norm}
\end{equation}
where $r(\cdot)$ is the absolute t-value as defined by equation (\ref{A5}).
The generalized Frobenius norm of a $D_{1}\times D_{2}$ t-matrix $W_\mathit{TM}$ is defined by
\begin{equation}
\|W_\mathit{TM}\|_{F} \doteq
\left(
{\sum\limits_{\alpha=1}^{D_{1}}\sum\limits_{\beta=1}^{D_{2}}r(W_{\mathit{TM},\alpha,\beta})^{2}}\right)^{1/2}.
\label{A7}
\end{equation}

In order to have a mechanism to connect t-matrices with matrices with elements in $\mathbb{C}$ or $\mathbb{R}$, the slices of a t-matrix are defined as follows.

(8) {\bf Slice of a t-matrix} :
Any t-matrix $X_\mathit{TM} \in C^{D_1\times D_2}$, organized as an array in $\mathbb{C}^{I_1\times \cdots \times I_N\times D_1\times D_2}$,
can be sliced
into
$\prod\nolimits_{n=1}^{N}I_{n}$ matrices in $\mathbb{C}^{D_1\times D_2}$, indexed by the multi-index $i$.
Let $X_\mathit{TM}(i) \in \mathbb{C}^{D_1\times D_2}$ be the $i$-th slice. The entries of $X_\mathit{TM}(i)$ are complex numbers in $\mathbb{C}$ given by
\[
(X_\mathit{TM}(i))_{\alpha,\beta} = X_{\mathit{TM},i, \alpha,\beta} \in \mathbb{C}~~
\]
for all indices $1\le i\le I ,
1\le \alpha\le D_1, 1\le \beta\le D_2$.

The t-vectors with a given dimension form an algebraic structure called a module over the ring $C$ \cite{Hungerford}. Modules are generalizations of vector spaces
\cite{Kilmer2013Third-TProduct003-0005}.
The t-vector whose entries are all equal to $Z_{T}$ is denoted by $Z_\mathit{TV}$, and called the zero t-vector. The next step is to define what is meant by a set of linearly independent t-vectors and what is meant by a full column rank t-matrix.

(9) {\bf Linear independence in t-vector module}:
The t-vectors in a subset  $\{X_{\mathit{TV}, 1},$ $X_{\mathit{TV}, 2},\cdots, X_{\mathit{TV}, K}  \}$
of a t-vector module are said to be linearly independent if the equation
$
\sum\nolimits_{k=1}^{K}\lambda_{T,k} \myconv X_{\mathit{TV},k} = Z_\mathit{TV}
$ holds true if and only if $\lambda_{T,k} = Z_{T}$, $1\le k\le K$.

If the t-vectors $X_{\mathit{TV},i}$, $1\le i\le K$, are linearly independent then they are said to have {a rank of $K$}. If the t-vectors $Y_{\mathit{TV},i}$ for $1\le i\le K'$ are linearly independent and span the same sub-module as the $X_{\mathit{TM},i}$ then $K = K'$. For further information see \cite{Hungerford}.

(10) {\bf Full column rank t-matrix}: A t-matrix is said to be of full column rank if
all its column t-vectors are linearly independent.

\subsection{T-matrix Analysis via the Fourier Transform}
\label{section:OperationsViaSlices}

The Fourier transform of the t-matrix $X_\mathit{TM} \in \mathit{C}^{D_1\times D_2}$ is the t-matrix in $C^{D_1\times D_2}$ given by equation (\ref{equation:fourier-transform-t-matrix}).

Many t-matrix computations can be carried out efficiently using the Fourier transform.
For example, any multiplication
$C_\mathit{TM} = X_\mathit{TM} \myconv Y_\mathit{TM} \in
\mathbb{C}^{I_1\times \cdots \times I_N\times D_1\times D_2}$, where $X_\mathit{TM} \in \mathbb{C}^{I_1\times \cdots \times I_N\times D_1\times Q}$,
$Y_\mathit{TM} \in \mathbb{C}^{I_1\times \cdots \times I_N\times Q\times D_2}$,  can be decomposed to $\prod\nolimits_{n=1}^{N}I_{n}$ matrix multiplications over the complex numbers, namely
\begin{equation}
F(C_\mathit{TM})_{i, \alpha,\beta} = \sum\nolimits_{\gamma=1}^{Q}F(X_\mathit{TM})_{i, \alpha,\gamma} \cdot F(Y_\mathit{TM})_{i, \gamma,\beta}
\end{equation}
for all indices $1\le i\le I, 1\le \alpha\le D_1, 1\le\beta\le D_2$.

The conjugate transpose
$X_\mathit{TM}^{\myhtrans} \in$ $\mathbb{C}^{I_1\times \cdots \times I_N\times D_2\times D_1}$ of any t-matrix $X_\mathit{TM} \in \mathbb{C}^{I_1\times \cdots \times I_N\times D_1\times D_2}$
can be decomposed to $\prod\nolimits_{n=1}^{N}I_{n}$ canonical conjugate transposes of matrices,
\begin{equation}
F(X_\mathit{TM}^{\myhtrans})_{i,\beta,\alpha} = \overline{F(X_{\mathit{TM}})_{i, \alpha,\beta}}
\end{equation}
for all indices $1\le i\le I,1\le\alpha\le D_1,1\le \beta\le D_2$.
Each slice of $F\left(I_\mathit{TM}^{(D)}\right)$ is the canonical identity matrix with elements in $\mathbb{C}$.

The Fourier transform decomposes a t-matrix computation such as multiplication to $\prod\nolimits_{n=1}^{N}I_{n}$ independent complex matrix computations in the Fourier domain. The $i$-th ($1 \le i \le I$)
computation involves only the $i$-th slices of the associated t-matrices. This fact underlies an approach for speeding-up t-matrix algorithms using parallel computations.
This independence of the data in the Fourier domain makes it possible to implement parallel computing using the so-called vectorization programming (also known as array programming), which is supported by many programming languages including MATLAB, R, NumPy, Julia, and Fortran.

\subsection{Pooling}

Sometimes, it is necessary to have a pooling mechanism to transform t-scalars to scalars
in $\mathbb{R}$ or $\mathbb{C}$. Given any t-scalar $X_{T} \in C$, its pooling result
$P(X_{T}) \in \mathbb{C}$ is defined by
\begin{equation}
P(X_T) =
(I_1\cdots I_N)^{-1}
\sum\limits_{i=1}^{I}X_{T,i} .
\label{equation:pooling-of-t-scalar}
\end{equation}

The pooling operation for t-matrices transforms each t-scalar entry to a scalar. More formally,
given any t-matrix $Y_\mathit{TM} \in C^{D_1\times D_2}$, its pooling result $P(Y_\mathit{TM})$ is
by definition the matrix in $\mathbb{C}^{D_1\times D_2}$ given by
\begin{equation}
\begin{aligned}
P(Y_\mathit{TM})_{\alpha,\beta} = P(Y_{\mathit{TM},\alpha,\beta}),
~1\le\alpha\le D_{1}, 1\le \beta\le D_{2}.
\end{aligned}
\label{equation:pooling-operation}
\end{equation}

The pooling of t-vectors is a special case of equation (\ref{equation:pooling-operation}).

\subsection{Generalized tensors}
\label{subsection:Generalizedt-scalars}

Generalized tensors, called g-tensors, generalize t-matrices and canonical tensors. The generalized tensors defined in this section are used to construct the higher order TSVD in Section \ref{section:THOSVD}. A g-tensor, denoted by
$X_\mathit{GT} \in C^{D_1\times D_2 \times \cdots \times D_M}$,
is a generalized tensor with t-scalar entries (i.e., an order-$M$ array of t-scalars).
Its t-scalar entries are indexed by $(X_\mathit{GT})_{\alpha_1,\cdots ,\alpha_M}$.
Then, a generalized mode-$k$ multiplication of $X_\mathit{GT}$, denoted by
$M_\mathit{GT} \doteq
X_\mathit{GT} ~\myconv_{k}~ Y_\mathit{TM}$ where $Y_\mathit{TM} \in C^{J\times D_k}$
and $1 \le k \le M$,
is a g-tensor in $C^{D_1\times \cdots \times D_{k-1} \times J \times D_{k+1} \times \cdots \times D_{M} }$ defined as follows.
\begin{equation}
\begin{gathered}[]
(M_\mathit{GT})_{\alpha_1, \cdots, \alpha_{k-1}, \beta, \alpha_{k+1}, \cdots, \alpha_M}
\\
=
\sum\limits_{\alpha_k = 1}^{D_k} (X_\mathit{GT})_{\alpha_1,\cdots,\alpha_{k-1},\alpha_k,\alpha_{k+1}\cdots, \alpha_{M}}
\myconv (Y_\mathit{TM})_{\beta, \alpha_k}
\end{gathered}
\label{equation:generalized-mod-k-multiplication}
\end{equation}

The {\bf generalized mode-$k$ flattening} of a g-tensor
$X_\mathit{GT} \in C^{D_1\times D_2 \times \cdots \times D_M}$
is an $(K_1, K_2)$-reshaping where $K_1 = \{k\}$ and $K_2 = \{1,\cdots,M\} \setminus \{k\}$.
The result is a t-matrix
in $C^{D_{k} \times D_k^{-1}\cdot{\prod\nolimits_{m=1}^{M} D_{m}}}$. Each column of the matrix is obtained by holding the indices in $K_{2}$ fixed and varying the index in $K_{1}$.

The generalized mode-$k$ multiplication defined in equation
(\ref{equation:generalized-mod-k-multiplication}) can also be expressed
in terms of unfolded g-tensors:
$$
M_{\mathit{GT}} =
X_\mathit{GT} ~\myconv_{k}~ Y_\mathit{TM} \;\;
\Leftrightarrow \;\;
M_{\mathit{GT}(k)} =
Y_\mathit{TM} \myconv
X_{\mathit{GT}(k)}
$$
where
$M_{\mathit{GT}(k)} \in C^{J\times (D_1 \cdots D_{k-1} D_{k+1} \cdots  D_M)}$
and
$X_{\mathit{GT}(k)} \in C^{D_k\times (D_1 \cdots D_{k-1} D_{k+1} \cdots  D_M)}$
are respectively the generalized mode-$k$ flattening of
the g-tensors $M_{\mathit{GT}}$ and $X_{\mathit{GT}}$.

An example of a generalized tensor (g-tensor) $X_\mathit{GT} $ $\in$ $C^{2\times 3\times 2}\equiv
\mathbb{C}^{3\times 3\times 2\times 3\times 2}$,
its mode-$k$ flattening,
and its mode-$2$ multiplication with a t-matrix $Y_\mathit{TM} \in C^{2\times 3}\equiv \mathbb{C}^{3\times 3\times 2\times 3}$ are given {in a supplementary file}.

\section{Tensor Singular Value Decomposition}
\label{section:generalized-SVD}

The singular value decomposition (SVD) is a well known factorization of real or complex matrices \cite{GolubVanLoan}. It generalizes the eigen-decomposition of positive semi-definite normal matrices to non-square and non-normal matrices. The SVD has a wide range of applications in data analytics, including computing the pseudo-inverse of a matrix, solving linear least squares problems, low-rank approximation and
linear and multi-linear component analysis.
A tensor version TSVD of the SVD is described in Section \ref{section:TSVD}, and then applied in Section \ref{section:THOSVD} to obtain a tensor version, THOSVD, of the Higher Order SVD (HOSVD). Further information about the TSVD can be found in \cite{Kilmer2011Factorization-TProduct001-0006} and \cite{liaoliang-00016}.

\subsection{TSVD: Tensorial SVD}
\label{section:TSVD}

\emph{Algorithm}.
A tensor version, TSVD, of the singular value decomposition is described in this section and then applied in Section \ref{section:THOSVD} to obtain a tensor version of the High Order SVD (HOSVD). See \cite{liaoliang-00016} and \cite{Kilmer2011Factorization-TProduct001-0006}.

Given a t-matrix $X_\mathit{TM} \in C^{D_1\times D_2}$, let $Q \doteq \min(D_1, D_2)$.
The TSVD of $X_\mathit{TM}$ yields
the following three t-matrices $U_\mathit{TM} \in C^{D_1 \times Q}$,
$S_\mathit{TM} \in C^{Q\times Q}$ and $V_\mathit{TM} \in C^{D_2\times Q}$, such that
\begin{equation}
X_\mathit{TM} = U_\mathit{TM} \myconv S_\mathit{TM} \myconv  V_\mathit{TM}^{\myhtrans}
\label{equation:tSVD}
\end{equation}
where $U_\mathit{TM}^{\myhtrans}\myconv U_\mathit{TM} =
V_\mathit{TM}^{\myhtrans}\myconv V_\mathit{TM} = I_\mathit{TM}^{(Q)}$,
$S_\mathit{TM} =
\operatorname{diag}(\lambda_{T, 1},\cdots,\lambda_{T, Q})$
and $\lambda_{T, 1},$ $\cdots,$ $\lambda_{T, Q} \in C$ are nonnegative, and satisfy
$F(\lambda_{T, 1})_{i}
\ge \cdots \ge
F(\lambda_{T, Q})_{i}\ge 0\;,~1\le i\le I.
$ The t-matrices $U_{\mathit{TM}}$ and $V_{\mathit{TM}}$ are generalizations of the orthogonal matrices in the SVD of a matrix with elements in $\mathbb{R}$ or $\mathbb{C}$.

Although it is possible to compute $U_\mathit{TM}$, $S_\mathit{TM}$ and
$V_\mathit{TM}$ in the spatial domain, it is preferable to organize the TSVD algorithm in the Fourier domain, because of the observation in Section \ref{FourierTransformOfAT-Scalar} that the Fourier transform converts the convolution product to the Hadamard product. The TSVD of $X_\mathit{TM}$ can be decomposed into $\prod\nolimits_{n=1}^{N}I_{n}$ SVDs of complex number matrices given by the slices of the Fourier transform $F(X_{\mathit{TM}})$. The t-matrices $U_\mathit{TM}$, $S_\mathit{TM}$ and $V_\mathit{TM}$ in equation (\ref{equation:tSVD}) are obtained in Algorithm \ref{algorithm:tSVD}.

\renewcommand{\algorithmicrequire}{\textbf{Inputs:}}
\renewcommand\algorithmicensure {\textbf{Outputs:}}
\begin{algorithm}[htb]
\caption{$(U_\mathit{TM}, S_\mathit{TM}, V_\mathit{TM}) = \mathtt{tsvd}(X_\mathit{TM})$}
\begin{algorithmic}[1]
\REQUIRE A t-matrix $X_\mathit{TM} \in C^{D_1\times D_2}$ as in equation (\ref{equation:tSVD})  and the t-scalar dimensions $I$.
\ENSURE The t-matrices $U_\mathit{TM} $, $S_\mathit{TM}$ and $V_\mathit{TM} $ as in equation (\ref{equation:tSVD})
\STATE  Compute the Fourier transformation~~$\tilde{X}_\mathit{TM}\leftarrow F(X_\mathit{TM})$
\FORALL {$1 \le i \le I$}
\STATE Compute the canonical SVD of $\tilde{X}_\mathit{TM}(i) \in \mathbb{C}^{D_1\times D_2}$
such that
\begin{equation}
\tilde{X}_\mathit{TM}(i) = U_\mathit{mat} \cdot S_\mathit{mat}\cdot  V_\mathit{mat}^{H}
\nonumber
\end{equation}
where
$U_\mathit{mat} \in \mathbb{C}^{D_1\times Q}$,
$S_\mathit{mat} \in \mathbb{C}^{Q\times Q}$,
$V_\mathit{mat} \in \mathbb{C}^{D_2\times Q}$,
$Q \doteq \min(D_1, D_2)$
and
$V_\mathit{mat}^{H}$ denotes the conjugate transpose of the complex matrix $V_\mathit{mat}$.
\STATE Assign the $i$-th slices  
$\tilde{U}_\mathit{TM}  \in C^{D_1\times Q}$,
$\tilde{S}_\mathit{TM} \in C^{Q\times Q}$  and
$\tilde{V}_\mathit{TM} \in C^{D_2\times Q}$ in the Fourier domain \newline
\vspace{-1em}
$$
\tilde{U}_\mathit{TM}(i) \leftarrow U_\mathit{mat},\;\;\;
\tilde{S}_\mathit{TM}(i) \leftarrow S_\mathit{mat}, \;\;\;
\tilde{V}_\mathit{TM}(i) \leftarrow V_\mathit{mat}.
\vspace{-0.5em}
$$
\ENDFOR
\STATE Compute the inverse transforms to obtain $U_\mathit{TM}$, $S_\mathit{TM}$ and $V_\mathit{TM}$\\
\vspace{-1em}
$$
U_\mathit{TM} \leftarrow F^{-1}(\tilde{U}_\mathit{TM}), \;
    S_\mathit{TM} \leftarrow F^{-1}(\tilde{S}_\mathit{TM}),  \;
    V_\mathit{TM} \leftarrow F^{-1}(\tilde{V}_\mathit{TM}).
\vspace{-0.5em}
$$
\STATE \textbf{return}
    $U_\mathit{TM}, \;
    S_\mathit{TM},  \;
    V_\mathit{TM}$.
\end{algorithmic}
\label{algorithm:tSVD}
\end{algorithm}

If $X_{\mathit{TM}}$ is defined over $\mathbb{R}$, then $U_{\mathit{TM}}$, $S_{\mathit{TM}}$ and $V_{\mathit{TM}}$ can be chosen such that they are defined over $\mathbb{R}$. It is sufficient to choose the slices $\tilde{U}_{\mathit{TM}}(i)$, $\tilde{S}_{\mathit{TM}}(i)$ and $\tilde{V}_{\mathit{TM}}(i)$ such that
$\tilde{U}_{\mathit{TM}}(i) = \overline{\tilde{U}}_{\mathit{TM}}(2-i)$,
$\tilde{S}_{\mathit{TM}}(i) = \overline{\tilde{S}}_{\mathit{TM}}(2-i)$ and
$\tilde{V}_{\mathit{TM}}(i) = \overline{\tilde{V}}_{\mathit{TM}}(2-i)$. When the t-scalar dimensions are given by $N = 1$, $I_{1}=1$,
TSVD reduces to the canonical SVD of a matrix
in $\mathbb{C}^{D_1\times D_2}$.
The properties of the SVD can be used to show that the t-matrix $S_\mathit{TM}$ in Algorithm \ref{algorithm:tSVD} is unique. The t-matrices $U_\mathit{TM}$ and $V_\mathit{TM}$ are not unique.

\emph{TSVD Approximation}.
{TSVD can be used to approximate data. Given a t-matrix $X_\mathit{TM} \in \mathit{C}^{D_1\times D_2}$,
let $Q \doteq \min(D_1, D_2)$ and let the TSVD of $X_{\mathit{TM}}$ be computed
as in equation (\ref{equation:tSVD}).
The low-rank approximation $\hat{X}_{\mathit{TM}}$ of $X_\mathit{TM}$
with rank of $r$ ($1 \le r \le Q$) is defined by
\begin{equation}
\hat{X}_\mathit{TM} = U_\mathit{TM} \myconv \hat{S}_\mathit{TM} \myconv V_\mathit{TM}^{\myhtrans} \;\;
\label{equation:TSVD-low-rank-approximation111}
\end{equation}
where $\hat{S}_\mathit{TM} = \operatorname{diag}(\lambda_{T, 1}, \cdots,\lambda_{T, r},
\underset{Q-r}{\underbrace{Z_{T}, \cdots, Z_{T}}} )$ and $\lambda_{T, 1},$ $\cdots,$ $\lambda_{T, r} \neq Z_T$.

When the t-scalar dimensions are given by $N = 1$, $I_{1} = 1$, equation (\ref{equation:TSVD-low-rank-approximation111}) reduces to the SVD {low-rank} approximation to a matrix in $\mathbb{C}^{D_1\times D_2}$.

Furthermore, we contend that the approximation $\hat{X}_\mathit{TM}$ computed as in equation
(\ref{equation:TSVD-low-rank-approximation111}) is the solution of the following optimization problem
\begin{equation}
\begin{aligned}
X_\mathit{TM}^\mathit{approx} = \mathop{\operatorname{argmin}}\nolimits_{Y_\mathit{TM} \in C^{D_1\times D_2}} \|X_\mathit{TM} - Y_\mathit{TM}\|_F \\
\text{subject to} \operatorname{rank}(Y_\mathit{TM}) \le r\cdot E_T
\end{aligned}
\label{equation:generalized-Eckart-theorem}
\end{equation}
where $\|\cdot\|_F$ denotes the generalized Frobenius norm of a t-matrix, which is a nonnegative t-scalar,
as defined in equation (\ref{A7}). The result $X_\mathit{TM}^\mathit{approx}$ generalizes the Eckart-Young-Mirsky theorem \cite{eckart1936approximation}.

{To have an optimization problem in the form of  (\ref{equation:generalized-Eckart-theorem}), the notation
$\operatorname{rank}(\cdot)$, i.e., the rank of a t-matrix, and $\min(\cdot)$, i.e., the minimization of a nonnegative t-scalar variable belonging to a subset of $S^\mathit{nonneg}$,
and the ordering relationship $\le$ between two nonnegative t-scalars
need to be defined.

These definitions generalize their canonical counterparts. The definitions and the generalized Eckart-Young-Mirsky theorem are discussed in an appendix.

\subsection{THOSVD: Tensor Higher Order SVD}
\label{section:THOSVD}

In multilinear algebra, the higher order singular value decomposition (HOSVD), also known as the orthogonal Tucker decomposition of a tensor, is a generalization of the SVD. It is commonly
used to extract directional information from multi-way arrays \cite{tucker1966some-00012,de2000multilinear-0007}.
The applications
of HOSVD include
data analytics \cite{vannieuwenhoven2012new-00013,taguchi2017tensor},
machine learning \cite{Vasilescu2002Human,vasilescu2002multilinear-00014,lu2008mpca-00010},
DNA and RNA analysis \cite{Omberg2007A,muralidhara2011tensor} and
texture mapping in computer graphics \cite{vasilescu2004tensortextures}.

On using the t-scalar algebra,
the HOSVD can be generalized further to obtain a tensorial HOSVD, called THOSVD.
The THOSVD is obtained by replacing the complex number elements of each multi-way array by t-scalar elements.
Based on the definitions of g-tensors in Section \ref{subsection:Generalizedt-scalars}, the THOSVD of
$X_\mathit{GT} \in C^{D_1\times D_2 \times \cdots \times D_M}$ is given by the following generalized mode-$k$ multiplications.
\begin{equation}
X_\mathit{GT} =  S_\mathit{GT} \myconv_{1}~ U_{\mathit{TM}, 1}
\myconv_{2}~ U_{\mathit{TM}, 2} \cdots
\myconv_{M}~ U_{\mathit{TM}, M}
\label{equation:THOSVD}
\end{equation}
where $S_\mathit{GT} \in C^{Q_1\times Q_2 \times \cdots \times Q_M}$ is called the core g-tensor, $U_{\mathit{TM}, k} \in C^{D_k\times Q_k}$ is the mode-$k$ factor t-matrix and $Q_k \doteq
\min(D_k, D_{k}^{-1}\prod\nolimits_{m=1}^{M}D_m)$ for $1\le k\le M$.

Given a g-tensor $X_\mathit{GT} \in C^{D_1\times D_2 \times \cdots \times  D_M}$, the THOSVD of $X_\mathit{GT}$, as in equation (\ref{equation:THOSVD}), is obtained in Algorithm \ref{algorithm:THOSVD}, using a strategy analogous to that of Tucker \cite{tucker1966some-00012} and
De Lathauwer et al. \cite{de2000multilinear-0007} for computing the HOSVD of a tensor with elements in $\mathbb{R}$ or $\mathbb{C}$.

\renewcommand{\algorithmicrequire}{\textbf{Input:}}
\renewcommand\algorithmicensure {\textbf{Outputs:}}
\begin{algorithm}[tb]
\caption{THOSVD}
\begin{algorithmic}[1]
\REQUIRE $X_\mathit{GT} \in C^{D_1\times D_2 \times \cdots \times D_M}$.
\ENSURE $U_{\mathit{TM}, 1},~ U_{\mathit{TM}, 2},~ \cdots,~ U_{\mathit{TM}, M}$ and  $S_\mathit{GT}$ as in equation (\ref{equation:THOSVD})
\FORALL {$1 \le k \le M$ }
\STATE Construct the generalized mode-$k$ flattening
$X_\mathit{GT(k)} \in C^{D_{k}\times D_{k}^{-1}\prod_{m=1}^{M} D_m}$.
\STATE $(U_{\mathit{TM}, k}, \mathtt{\sim}, \mathtt{\sim}) \leftarrow \mathtt{tsvd}(X_\mathit{GT(k)})$
\ENDFOR
\STATE $S_\mathit{GT} \leftarrow X_\mathit{GT} \myconv_{1}~ U_{\mathit{TM}, 1}^{\myhtrans}
\myconv_{2}~
U_{\mathit{TM}, 2}^{\myhtrans}
\cdots
\myconv_{M}~ U_{\mathit{TM}, M}^{\myhtrans}$
\STATE \textbf{return}
    $U_{\mathit{TM}, 1},~ U_{\mathit{TM}, 2},~ \cdots,~ U_{\mathit{TM}, M}$ and  $S_\mathit{GT}$.
\end{algorithmic}
\label{algorithm:THOSVD}
\end{algorithm}

Note that THOSVD generalizes the HOSVD for canonical tensors, TSVD for t-matrices,
and SVD for canonical matrices.
Many SVD and HOSVD based algorithms can be generalized by TSVD and THOSVD, respectively.

\section{Tensor Based Algorithms}
\label{section:image-analysis-SVD}

Three tensor based algorithms are proposed. They are Tensorial Principal Component Analysis (TPCA), Tensorial Two-Dimensional Principal Component Analysis (T2DPCA) and Tensorial Grassmannian Component Analysis (TGCA). TPCA and T2DPCA are generalizations of the well-known algorithms PCA and 2DPCA \cite{Yang2004TwoD}. TGCA is a generalization of the recent GCA algorithm \cite{Harandi2015Sparse,Harandi2015Extrinsic}.
It is possible to generalize many other linear or multi-linear algorithms using similar methods.

\subsection{TPCA: Tensorial Principal Component Analysis}

Principal Component Analysis (PCA) is a well known algorithm for extracting
the prominent components of observed vectors. PCA is generalized to TPCA in a straightforward manner.
Let $X_{\mathit{TV}, 1}, \cdots, X_{\mathit{TV}, K} \in C^{D}$ be $K$ given t-vectors. Then, the covariance-like t-matrix $G_\mathit{TM} \in C^{D\times D}$  is defined by
\begin{equation}
G_\mathit{TM} = \frac{1}{K-1}\sum\limits_{k=1}^{K} (X_{\mathit{TV}, k} - \bar{X}_\mathit{TV}  ) \myconv
(X_{\mathit{TV}, k} - \bar{X}_\mathit{TV}  )^{\myhtrans}
\end{equation}
where $\bar{X}_\mathit{TV} = (1/K)~ \sum\nolimits_{k=1}^{K} X_{\mathit{TV}, k}$. It is not difficult to verify that
$G_\mathit{TM}$ is Hermitian, namely
$G_\mathit{TM}^{\myhtrans} = G_\mathit{TM}$.

The t-matrix $U_\mathit{TM} \in C^{D\times D}$ is computed from the TSVD of $G_\mathit{TM}$ as in Algorithm \ref{algorithm:tSVD}. Then, given any t-vector $Y_\mathit{TV} \in C^{D}$, its feature t-vector $Y^\mathit{feat}_\mathit{TV} \in C^{D}$
is defined by
\begin{equation}
Y^\mathit{feat}_\mathit{TV} = U_\mathit{TM}^{\myhtrans} \myconv
(Y_\mathit{TV} - \bar{X}_\mathit{TV}) \;\;.
\end{equation}
To reduce $Y^\mathit{feat}_\mathit{TV}$ from a t-vector in $C^{D}$ to a t-vector in $C^{d}$ ($D>d$),
simply discard the last $(D-d)$ t-scalar entries of $Y^\mathit{feat}_\mathit{TV}$.

In algebraic terminology, the column t-vectors of $U_\mathit{TM}$ span a linear sub-module of t-vectors, which is a generalization of a vector subspace \cite{Braman2010Third-TProduct002-0001}.
In this sense, each t-scalar entry of $Y^\mathit{feat}_\mathit{TV}$ is a generalized coordinate of the projection of the t-vector $(Y_\mathit{TV} - \bar{X}_\mathit{TV} )$ onto the sub-module. The
{low-rank} reconstruction
$Y_\mathit{TV}^\mathit{rec}\in C^{D}$ with the parameter $d$  is given by
\begin{equation}
Y_\mathit{TV}^\mathit{rec} = \mybrace{U_\mathit{TM}}_{:, 1:d}\myconv \mybrace{Y^\mathit{feat}_\mathit{TV}}_{1:d} +\bar{X}_\mathit{TV}
\label{equation:reconstruction-of-TPCA}
\end{equation}
where $\mybrace{U_\mathit{TM}}_{:, 1:d} \in C^{D\times d}$ denotes the t-matrix containing the first $d$ t-vector columns of $U_\mathit{TM} \in C^{D\times D}$
and $\mybrace{Y^\mathit{feat}_\mathit{TV}}_{1:d} \in C^{d}$ denotes the t-vector containing the first $d$ t-scalar entries of $Y^\mathit{feat}_\mathit{TV} \in C^{D}$.

Note that PCA is a special case of TPCA. When the t-scalar dimensions are given by $N = 1$, $I_{1}=1$, TPCA reduces to PCA.

\subsection{T2DPCA: Tensorial Two-dimensional Principal Component Analysis}
The algorithm 2DPCA is an extension of PCA proposed by Yang et al. \cite{Yang2004TwoD}
for analysing the principal components of matrices.
Although 2DPA is written in a non-centred row-vector oriented form in the original paper \cite{Yang2004TwoD}, it is rewritten here in a centred column-vector oriented form, which is consistent with the formulation of PCA.  The centred column-vector oriented form of 2DPCA is chosen for discussing its generalization to T2DPCA (Tensorial 2DPCA).

Similar to TPCA, T2DPCA also finds sub-modules, but they are obtained by analysing t-matrices.
Let $X_{\mathit{TM}, 1},$ $\cdots,$ $X_{\mathit{TM}, K} \in C^{D_1\times D_2}$ be the $K$ observed t-matrices. Then, the Hermitian covariance-like t-matrix
$G_\mathit{TM}\in C^{D_1\times D_1}$ is given by
\begin{equation}
G_\mathit{TM} = \frac{1}{K-1} \sum\limits_{k=1}^{K} (X_{\mathit{TM}, k} - \bar{X}_\mathit{TM}  ) \myconv
(X_{\mathit{TM}, k} - \bar{X}_\mathit{TM})^{\myhtrans}
\end{equation}
where $\bar{X}_\mathit{TM} = (1/K)~ \sum\nolimits_{k=1}^{K} X_{\mathit{TM}, k}$.

Then, the t-matrix $U_\mathit{TM} \in C^{D_1\times D_1}$ is computed from the TSVD of $G_\mathit{TM}$
as in Algorithm \ref{algorithm:tSVD}.
Given any t-matrix $Y_\mathit{TM} \in C^{D_1\times D_2}$, its feature t-matrix $Y^\mathit{feat}_\mathit{TM} \in C^{D_1\times D_2}$
is a centred t-matrix projection (i.e., a collection of
centred column t-vector projections) on the module spanned by $U_\mathit{TM}$, namely
\begin{equation}
Y^\mathit{feat}_\mathit{TM} = U_\mathit{TM}^{\myhtrans} \myconv
(Y_\mathit{TM} - \bar{X}_\mathit{TM}) \;.
\end{equation}

To reduce $Y^\mathit{feat}_\mathit{TM}$ from a t-matrix in
$C^{D_1\times D_2}$ to a t-matrix in $C^{d\times D_2}$ ($D_1>d$), simply discard the last $(D_1-d)$ row t-vectors of $Y^\mathit{feat}_\mathit{TM}$.

The T2DPCA reconstruction with the parameter $d$ is given by $Y_\mathit{TM}^\mathit{rec} \in C^{D_1\times D_2}$ as follows.
\begin{equation}
Y_\mathit{TM}^\mathit{rec} = {U_\mathit{TM}}_{:, 1:d} \myconv \mybrace{Y^\mathit{feat}_\mathit{TM}}_{1:d, :}
+ \bar{X}_\mathit{TM}
\label{equation:T2DPCA-reconstruction}
\end{equation}
where
$\mybrace{U_\mathit{TM}}_{:, 1:d} \in C^{D_1\times d}$  denotes the t-matrix containing the first $d$ column t-vectors of $U_\mathit{TM}$
and
$\mybrace{Y_\mathit{TM}^\mathit{feat}}_{1:d, :}$ $\in$ $C^{d\times D_2}$ denotes the t-matrix containing the first $d$ row t-vectors of $Y^\mathit{feat}_\mathit{TM}$.

When the t-scalar dimensions are given by $N = 1$, $I_{1}=1$, T2DPCA reduces to 2DPCA.
In addition, TPCA is a special case of T2DPCA.
When $D_2 = 1$, T2DPCA reduces to TPCA. Furthermore, when $N=1, I_1 = 1$ and $D_2 = 1$, T2DPCA reduces to PCA.

\subsection{TGCA: Tensorial Grassmannian Component Analysis}

A t-matrix algorithm which generalizes the recent algorithm for Grassmannian Component Analysis (GCA) is proposed. An example of GCA an be found in \cite{Harandi2015Extrinsic},     where it forms part of an algorithm for sparse coding on Grassmann manifolds.
In this section GCA is extended to its generalized version called TGCA (Tensorial GCA).

In TGCA, each measurement is a set of t-vectors organized into a ``thin'' t-matrix, with the number of rows larger than the number of columns. Let $X_{\mathit{TM}, 1},$ $\cdots,$ $X_{\mathit{TM}, K}$ $\in$ $C^{D\times d}$ ($D > d$) be the observed t-matrices.
Then, the t-vector columns of each t-matrix are first orthogonalized. Using the t-scalar algebra, it is straightforward to generalize the classical Gram-Schmidt orthogonalization process for t-vectors.
The TSVD can also be used to orthogonalise a set of t-vectors.  In GCA and TGCA, the choice of orthogonalization algorithm doesn't matter as long as the algorithm is consistent for
all sets of vectors and t-vectors.

Given a t-matrix $Y_\mathit{TM} \in C^{D\times d}$, let  $\dot{Y}_\mathit{TM} \in C^{D\times d}$ be the corresponding unitary orthogonalized t-matrix  (namely, $\dot{Y}_\mathit{TM}^{\myhtrans} \circ \dot{Y}_\mathit{TM} = I_\mathit{TM}^{(d)}$ ) computed from  $Y_\mathit{TM}$.
Let $(Y_\mathit{TM})_{:, k}$ be the $k$-th column t-vector of $Y_\mathit{TM}$ and let
$(\dot{Y}_\mathit{TM})_{:, k}$ be the $k$-th column t-vector of $\dot{Y}_\mathit{TM}$ for $1\le k\le d$.
The generalized Gram-Schmidt orthogonalization is given
by Algorithm \ref{algorithm:generalized-Schmidt}.

\renewcommand{\algorithmicrequire}{\textbf{Input:}}
\renewcommand\algorithmicensure {\textbf{Outputs:}}
\begin{algorithm}[tb]
\caption{Generalized Gram-Schmidt orthogonalization}
\begin{algorithmic}[1]
\REQUIRE $Y_{\mathit{TM}} \in C^{D\times d}$.
\ENSURE $\dot{Y}_\mathit{TM} \in C^{D\times d}$ satisfying $\dot{Y}_\mathit{TM}^{\myhtrans} \circ \dot{Y}_\mathit{TM} = I_\mathit{TM}^{(d)}$.
\STATE $(\dot{Y}_\mathit{TM})_{:, 1} \leftarrow
{\|(Y_\mathit{TM})_{:, 1}\|_{F}^{-1}  } \circ (Y_\mathit{TM})_{:, 1}$
\FORALL {$2 \le k \le d$ }
\STATE $(\dot{Y}_\mathit{TM})_{:, k} \leftarrow (Y_\mathit{TM})_{:, k}$
\FORALL {$1 \le j \le k -1$ }
\STATE $(\dot{Y}_\mathit{TM})_{:, k} \leftarrow (\dot{Y}_\mathit{TM})_{:, k} -
{\langle (\dot{Y}_\mathit{TM})_{:, k}, (\dot{Y}_\mathit{TM})_{:, j} \rangle } \circ (\dot{Y}_\mathit{TM})_{:, j}$
\ENDFOR
\STATE $(\dot{Y}_\mathit{TM})_{:, k} \leftarrow
{\| (\dot{Y}_\mathit{TM})_{:, k}\|_F^{-1}} \circ {(\dot{Y}_\mathit{TM})_{:, k}}$
\ENDFOR
\STATE \textbf{return} $\dot{Y}_\mathit{TM}$
\end{algorithmic}
\label{algorithm:generalized-Schmidt}
\end{algorithm}

\newcommand\dotX[1]{\dot{X}_{\mathit{TM},#1} }

Let $\dotX{k}  \in C^{D\times d}$ be
the unitary orthogonalized t-matrices computed from
$X_{\mathit{TM}, k}$ for $1\le k\le K$.
Then,
for $1 \le k,k' \le K$, the $(k,k')$ t-scalar entry of the symmetric
t-matrix $G_\mathit{TM} \in C^{K\times K}$ is nonnegative and given by
\begin{equation}
(G_\mathit{TM})_{k,k'} = \|
\dotX{k}^{\myhtrans} \myconv
\dotX{k'}\|_F^{2}
\;,\;\;1\le k, k'\le K
\label{equation:GTM-TGCA}
\end{equation}
where $\|\cdot\|_{F}$ is the generalized Frobenius norm of a t-matrix, as defined by equation (\ref{A7}).

Given any query t-matrix sample $Y_\mathit{TM} \in C^{D\times d}$, let  $\dot{Y}_\mathit{TM} \in C^{D\times d}$ be the corresponding unitary orthogonalized t-matrix computed from  $Y_\mathit{TM}$. Then,
the
$k$-th t-scalar entry of $K_\mathit{TV} \in C^{K}$
is computed as follows.
\begin{equation}
(K_\mathit{TV})_k =
\|
\dot{Y}_\mathit{TM}^{\myhtrans} \myconv
\dotX{k}\|_F^{2}\;\;
, \;\;
1\le k\le K.
\end{equation}

Since $G_\mathit{TM}$, computed as in equation (\ref{equation:GTM-TGCA}), is symmetric, the TSVD of $G_\mathit{TM}$ has the following form
\begin{equation}
G_\mathit{TM} = U_\mathit{TM} \myconv S_\mathit{TM} \myconv  U_\mathit{TM}^{\myhtrans}\;.
\end{equation}

\newcommand\mmydiag[2]{\operatorname{diag}(#1,\cdots,#2) }

Furthermore, if it is assumed that the diagonal entries
$S_\mathit{TM} \doteq \mmydiag{\lambda_{T, 1}}{\lambda_{T, K}}$
are all strictly positive, then the multiplicative inverse of
$\lambda_{T,k}$ exists for $1\le k\le K$. The t-matrix
$S_\mathit{TM}^{1/2} \doteq \mmydiag{\sqrt{\lambda_{T, 1}}}{\sqrt{\lambda_{T, K}}}$
is called the t-matrix square root of $S_\mathit{TM}$ and the t-matrix
$S_\mathit{TM}^{-1/2} \doteq \mmydiag{\frac{E_T}{\sqrt{\lambda_{T, 1}}}}{
\frac{E_T}{\sqrt{\lambda_{T, K}}}}$
is called the inverse t-matrix of $S_\mathit{TM}^{1/2}$.

Thus,
the features of the t-matrix sample $Y_\mathit{TM} \in C^{D\times d}$ are given by the t-vector $Y_\mathit{TV}^\mathit{feat} \in C^{K}$ as
\begin{equation}
Y_\mathit{TV}^\mathit{feat} = S_\mathit{TM}^{-{1}/{2}} \myconv U_\mathit{TM}^{\myhtrans} \myconv
K_\mathit{TV}
\end{equation}
and the features of the $k$-th measurement $X_{\mathit{TM}, k}$ are given by the t-vector $X_{\mathit{TV}, k}^\mathit{feat}$ as follows.
\begin{equation}
X_{\mathit{TV}, k}^\mathit{feat} = S_\mathit{TM}^{-{1}/{2}} \myconv U_\mathit{TM}^{\myhtrans} \myconv \mybrace{G_\mathit{TM}}_{:, k}
\;\;,
1 \le k \le K
\label{equation:TGCA-t-vector-forX}
\end{equation}
where $\mybrace{G_\mathit{TM}}_{:, k}$ denotes $k$-th t-vector column of $G_\mathit{TM}$. It is not difficult to verify that $S_\mathit{TM}^{-{1}/{2}} \myconv U_\mathit{TM}^{\myhtrans} \myconv G_\mathit{TM} \equiv S_\mathit{TM}^{{1}/{2}} \myconv U_\mathit{TM}^{\myhtrans}$.
This yields the following compact form for $X_{\mathit{TV},k}^\mathit{feat}$.
\begin{equation}
X_{\mathit{TV}, k}^\mathit{feat} = \mybrace{S_\mathit{TM}^{{1}/{2}} \myconv U_\mathit{TM}^{\myhtrans}}_{:, k}
\;\;,
1 \le k \le K
\label{equation:TGCA-t-vector-forX-second}
\end{equation}
where $\mybrace{S_\mathit{TM}^{{1}/{2}} \myconv U_\mathit{TM}^{\myhtrans}}_{:, k}$ denotes the $k$-th t-vector column of the t-matrix $(S_\mathit{TM}^{{1}/{2}} \myconv U_\mathit{TM}^{\myhtrans})$. The equation (\ref{equation:TGCA-t-vector-forX-second}) is more efficient in computations than
equation (\ref{equation:TGCA-t-vector-forX}).

The dimension of a TGCA feature t-vector is reduced from $K$ to $K'$  ($K > K'$) by discarding the last $(K-K')$ t-scalar entries.
It is noted that GCA is a special case of TGCA when the dimensions of the t-scalars are given by $N = 1$, $I_{1} = 1$.

\section{Experiments}
\label{section:experiments}

The results obtained from TSVD, THOSVD, TPCA, T2DPCA, TGCA and their precursors are compared in applications to {low-rank} approximation in Section \ref{section:low-rank-approximation}, reconstruction
in Section \ref{section:Reconstruction} and supervised classification of images in Section \ref{section:Classification}.

{In these experiments ``vertical'' and ``horizontal'' comparisons between generalised algorithms and the corresponding canonical algorithms are made.

In a ``vertical'' experiment, tensorized data is obtained from the canonical data in $3\times 3$ neighborhoods. The associated t-scalar is a $3\times 3$ array. To make the vertical comparison fair, we put the central slices of a generalized result into the original canonical form and then compare it with the result of the associated canonical algorithm.

In a ``horizontal'' comparison, a generalized order-$N$ array of order-two t-scalars is equivalent to a canonical order-$(N+2)$ array of scalars. Therefore, a generalized algorithm based on order-$N$ arrays of order-two t-scalars is compared with a canonical algorithm based on order-$(N+2)$ arrays of scalars.
}

\subsection{{Low-rank} Approximation}
\label{section:low-rank-approximation}
TSVD approximation is computed as in equation (\ref{equation:TSVD-low-rank-approximation111}).
THOSVD approximation generalizes low-rank approximation by TSVD and low-rank approximation by HOSVD. To simplify the calculations, the approximation is obtained for a g-tensor $X_{\mathit{GT}}$ in $C^{D_{1}\times D_{2}\times D_{3}}$. Let
$Q_k \doteq \min(D_k, D_{k}^{-1}D_1D_2D_3)$ for $k = 1, 2, 3$. The THOSVD of $X_\mathit{GT}$ yields
\begin{equation}
X_\mathit{GT} =  S_\mathit{GT} \myconv_{1}~ U_{\mathit{TM}, 1}
\myconv_{2}~ U_{\mathit{TM}, 2}
\myconv_{3}~ U_{\mathit{TM}, 3}
\end{equation}
where $U_{\mathit{TM}, k} \in C^{D_k \times Q_k}$ for $k=1,2,3$
and $S_\mathit{TM} \in C^{Q_1\times Q_2\times Q_3}$.

The low-rank approximation $\hat{X}_\mathit{GT} \in C^{D_1\times D_2\times  D_3}$
to $X_{\mathit{GT}}$ and with multilinear rank tuple $(r_1, r_2, r_3)$,
($1 \le r_k \le Q_k$ for all $ k = 1,2, 3$),
is computed as in equation (\ref{THOSVD-low-rank-approximation}),
where
$(U_{\mathit{TM}, k})_{:, 1: r_k}$ denotes the t-matrix containing the first
$r_k$
t-vector columns of $U_{\mathit{TM}, k}$
for $k = 1,2,3$
and $(S_\mathit{GT})_{1: r_1, 1: r_2, 1: r_3} \in C^{r_1\times r_2\times r_3}$
denotes the g-tensor containing the first $r_1\times r_2 \times r_3$ t-scalar entries of $S_\mathit{GT}$.

\begin{equation}
\begin{aligned}
\hat{X}_\mathit{GT} =
(S_\mathit{GT})_{1: r_1, 1: r_2, 1: r_3} \myconv_{1}~&  (U_{\mathit{TM}, 1})_{:, 1: r_1}
\myconv_{2}~
(U_{\mathit{TM}, 2})_{:, 1: r_2}
\myconv_{3}~
(U_{\mathit{TM}, 3})_{:, 1: r_3} \;\;.
\label{THOSVD-low-rank-approximation}
\end{aligned}
\end{equation}

When the t-scalar dimensions are given by $N = 1$, $I_{1} = 1$, equation (\ref{THOSVD-low-rank-approximation}) reduces to the
HOSVD {low-rank} approximation of a tensor in $\mathbb{C}^{D_1 \times D_2  \times D_3}$. When the g-tensor dimension $D_3 = 1$, equation (\ref{THOSVD-low-rank-approximation}) reduces to the SVD {low-rank} approximation of a canonical matrix in $\mathbb{C}^{D_1\times D_2}$.

\subsubsection{{TSVD versus SVD --- A ``Vertical'' Comparison}}

The low-rank approximation performances of TSVD and SVD are compared. In the experiment, the test sample is the $512\times 512\times 3$ RBG Lena image downloaded from
Wikipedia.\footnote{\url{https://en.wikipedia.org/wiki/Lenna}}.

For the SVD low-rank approximations, the RGB Lena image is split into three $512\times 512$ monochrome images. Each monochrome image is analyzed using the SVD.
{The three extracted monochrome Lena images are
order-two arrays in $\mathbb{R}^{512\times 512}$.}
Each monochrome Lena image is tensorized to produce a t-image
(a generalized monochrome image) in $R^{512\times 512} \equiv \mathbb{R}^{3 \times 3\times  512\times  512}$. In the tensorized version of the image each pixel value is replaced by a $3\times 3$ square of values obtained from the $3\times 3$ neighborhood  of the pixel. Padding with $0$ is used where necessary at the boundary of the image.

To evaluate the TSVD approximations in a manner relevant to the SVD approximations, upon obtaining a
t-image approximation $\hat{X}_\mathit{TM} \in \mathbb{R}^{3\times 3\times 512\times 512}$, the part $\hat{X}_\mathit{MT}(i)|_{i = (2,2)} \in \mathbb{R}^{512\times 512}$, i.e. the central slice of the TSVD approximation, is used for comparisons.

Given an array $X$ of any order over the real numbers $\mathbb{R}$, let $\hat{X}$ be an approximation to $X$.  Then, the PSNR (Peak Signal-to-Noise Ratio) for $\hat{X}$ is defined as in \cite{Almohammad2010StegoImage} by
\begin{equation}
\mathit{PSNR} = 20 \log_{10} \frac{\mathit{MAX}\cdot \sqrt{N^\mathit{entry}}}{\|X -\hat{X} \|_F}
\label{equation:PSNR}
\end{equation}
where $N^\mathit{entry}$ denotes the number of real number entries of $X$, $\|X- \hat{X}\|_F$ is the canonical Frobenius norm of the array $(X-\hat{X})$ and $\mathit{MAX}$ is the maximum possible value of the entries of $X$.
In all the experiments, $\mathit{MAX} = 255$. 

Figure \ref{fig:lena-approximation} shows the PSNR curves of the SVD and TSVD approximations as functions of the rank of $\hat{X}$.
It is clear that the PSNR of the TSVD approximation is consistently higher than that of SVD approximation. When the rank $r = 500$, the PSNRs of TSVD and SVD differ by more than than $37$ dBs.

\begin{figure*}[htb]
\begin{center}
\includegraphics[width=0.7\textwidth]{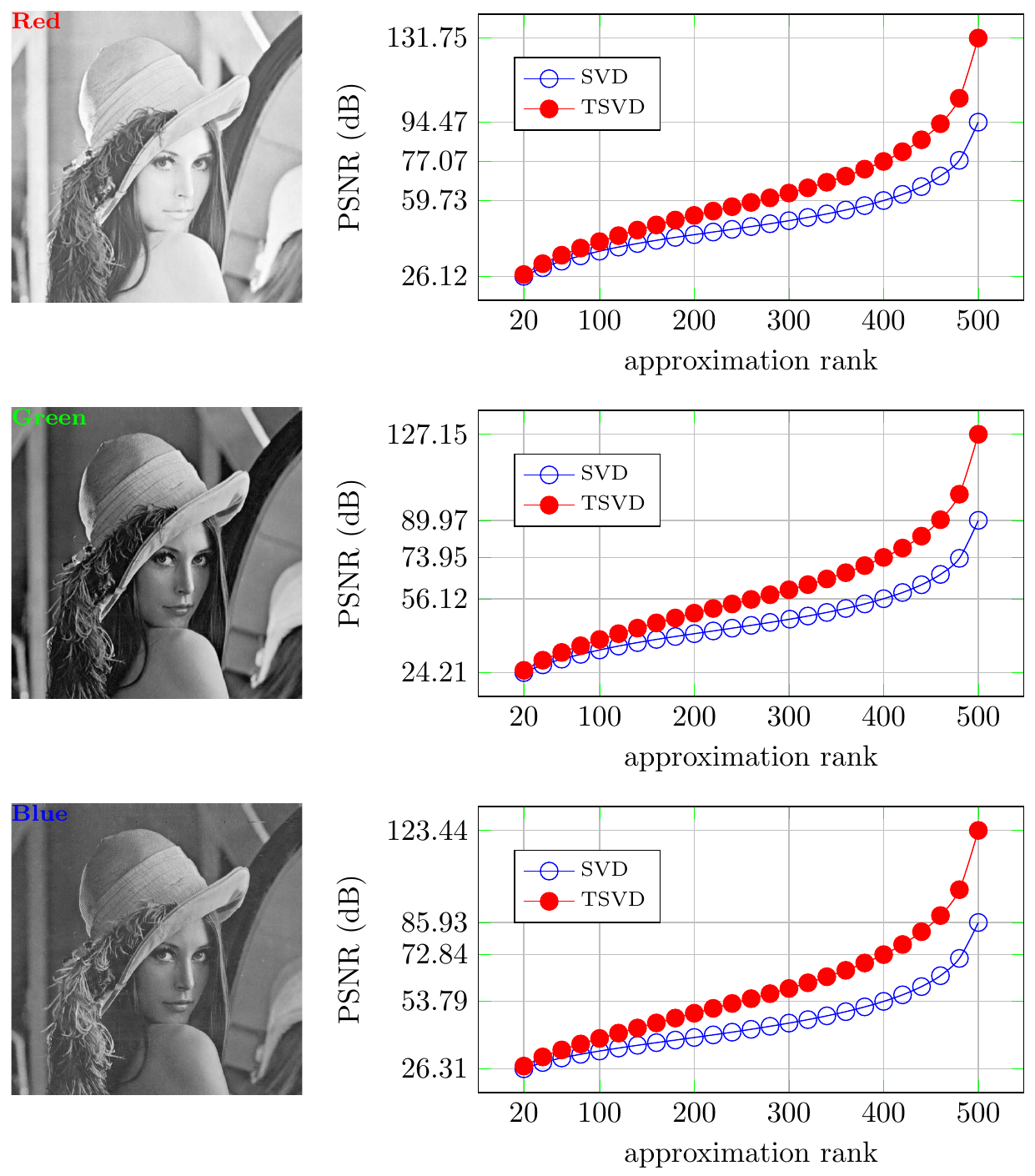}
\end{center}
\caption{A ``vertical'' comparison of low-rank approximations by SVD and TSVD for each monochrome Lena image.~
First column: Monochrome images extracted from the RGB Lena image.~
Second column: PSNR curves of SVD/TSVD approximation on/for each monochrome image.
}
\label{fig:lena-approximation}
\end{figure*}

\subsubsection{TSVD versus HOSVD --- A ``Horizontal'' Comparison}
Given a monochrome Lena image as an order-two array in $\mathbb{R}^{512\times 512}$ and its tensorized form as an order-four array in $\mathbb{R}^{3\times 3\times 512\times 512}$, TSVD yields an approximation array in $\mathbb{R}^{3\times 3\times 512\times 512}$.
Since the HOSVD is applicable to order-four arrays
in $\mathbb{R}^{3\times 3\times 512\times 512}$, we give a ``horizontal'' comparison of the performances of TSVD and HOSVD.

More specifically, given a generalized monochrome Lena image
$X_\mathit{TM} \equiv X \in C^{512\times 512} \equiv$ $\mathbb{R}^{3\times 3 \times 512\times 512}$ and a specified rank $r$, the TSVD approximation yields
a t-matrix $\hat{X}_\mathit{TM} \in C^{512\times 512} \equiv \mathbb{R}^{3\times 3 \times 512\times 512}$, which is computed as in equation (\ref{equation:TSVD-low-rank-approximation111}) with $D_1 = 512$ and $D_2 = 512$.

Let the HOSVD of $X \in \mathbb{R}^{3\times 3\times 512\times 512}$ be
$ X = S ~\times_1~ U_1 ~\times_2~ U_2 ~\times_3~ U_3 ~\times_4~ U_4
$ where $S \in \mathbb{R}^{3\times 3 \times 512\times 512} $ denotes the core tensor, and
$U_1 \in \mathbb{R}^{3\times 3}$,
$U_2 \in \mathbb{R}^{3\times 3}$,
$U_3 \in \mathbb{R}^{512\times 512}$,
$U_4 \in \mathbb{R}^{512\times 512}$ are all orthogonal matrices.
Then, to give a ``horizontal'' comparison with the TSVD approximation $\hat{X}_\mathit{TM}$ with
rank $r$,
the HOSVD approximation $\hat{X} \in \mathbb{R}^{3\times 3\times  512\times  512}$ is given by the multi-mode product
\begin{equation}
\begin{aligned}
\hat{X} =
(S)_{:, :, 1:r, 1:r} \times_1~ U_1 \times_2~ U_2 &~\times_3~ (U_3)_{:, 1: r} 
~\times_4~ (U_4)_{:, 1:r}  \;\;.
\end{aligned}
\end{equation}

The PSNRs TSVD and HOSVD are computed as in equation (\ref{equation:PSNR}) with $\mathit{MAX} = 255$ and $N^\mathit{entry} = 3\times 3\times 512\times 512 =  2359296$.

For each of the generalized monochrome Lena images (respectively marked by the channel type ``red'', ``green'' and ``blue''), as a ${3\times 3\times 512\times 512}$ real number array,
the PSNRs of TSVD and HOSVD are given in Figure \ref{fig:lapproximation_on_each_generalized_monochrome-lena-image}.

As rank $r$ is varied, the PSNR of TSVD approximation is always higher than that of
the corresponding HOSVD approximation. When rank $r$ is equal to 500, the PSNRs of TSVD and HOSVD approximations differ significantly.

\begin{figure*}[htb]
\begin{center}
\includegraphics[width=0.8\textwidth]{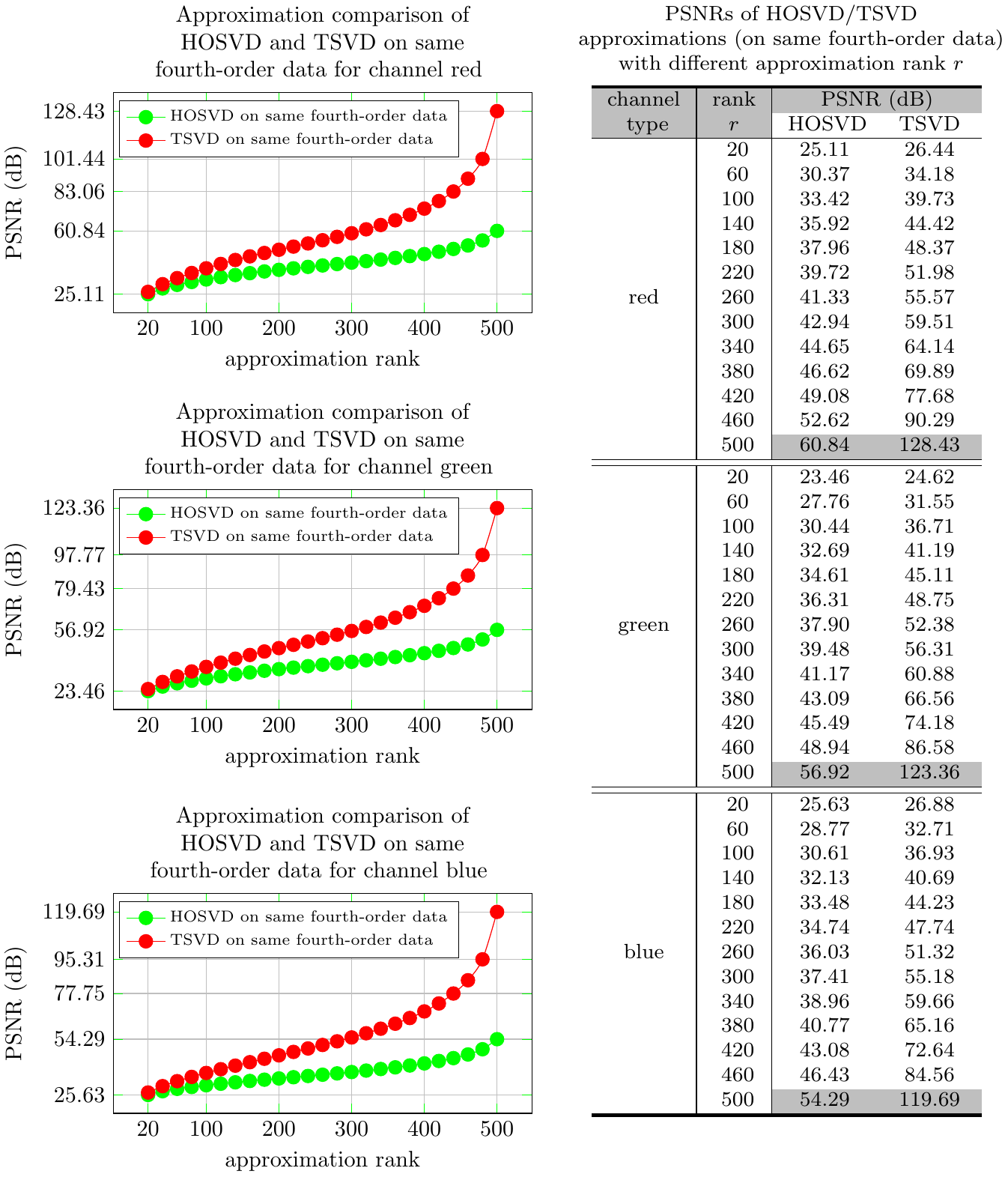}
\end{center}
\caption{A  ``horizontal'' comparison of low-rank approximations by HOSVD and TSVD on
each generalized monochrome Lena image, as an fourth-order real number array in $\mathbb{R}^{3\times 3\times 512\times 512}$.~
First column: PSNR curves, over rank $r$, of HOSVD/TSVD approximations  on each generalized
monochrome Lena image.~
Second column: Some quantitative PSNRs of HOSVD/TSVD approximations with rank $r$.
}
\label{fig:lapproximation_on_each_generalized_monochrome-lena-image}
\end{figure*}

\subsubsection{THOSVD versus HOSVD --- A ``Vertical'' Comparison}
The low-rank approximation performances of THOSVD and HOSVD are compared.
For the HOSVD approximations the RGB Lena image, which is a tensor in $\mathbb{R}^{512\times 512\times 3}$, is used as the test sample. For the THOSVD the $3\times 3$ neighborhood (with zero-padding) strategy is used to tensorize each real number entry of the RGB Lena image. The obtained t-image $X_\mathit{GT}$ is a g-tensor in ${R}^{512\times 512\times 3}$, i.e., an order-five array in $\mathbb{R}^{3\times 3\times 512\times 512\times 3}$.

To give a ``vertical'' comparison,
on obtaining an approximation $\hat{X}_\mathit{GT}$ $\in$ $\mathbb{R}^{3\times 3\times 512\times 512\times 3}$, we compare 
$\hat{X}_\mathit{GT}(i)|_{i = (2, 2)}$ $\in$ $\mathbb{R}^{512\times 512\times 3}$, i.e., the central slice of the THOSVD approximation, with the HOSVD approximation on the RGB Lena image.

Figure \ref{fig:lena-THOSV-HOSVD-approximation} gives
a ``vertical'' comparison of the PSNR maps of THOSVD and HOSVD approximations and the
tabulated PSNRs for some representative multilinear rank tuples $(r_1, r_2, r_3)$.
It shows the PSNR of the THOSVD approximation is consistently higher than the PSNR of the HOSVD approximation. When $(r_1, r_2, r_3) = (500, 500, 3)$, the approximations obtained by THOSVD and HOSVD differ by $30.29$ dB in their PSNR values.

\begin{figure*}[htb]
\begin{center}
\includegraphics[width=0.9\textwidth]{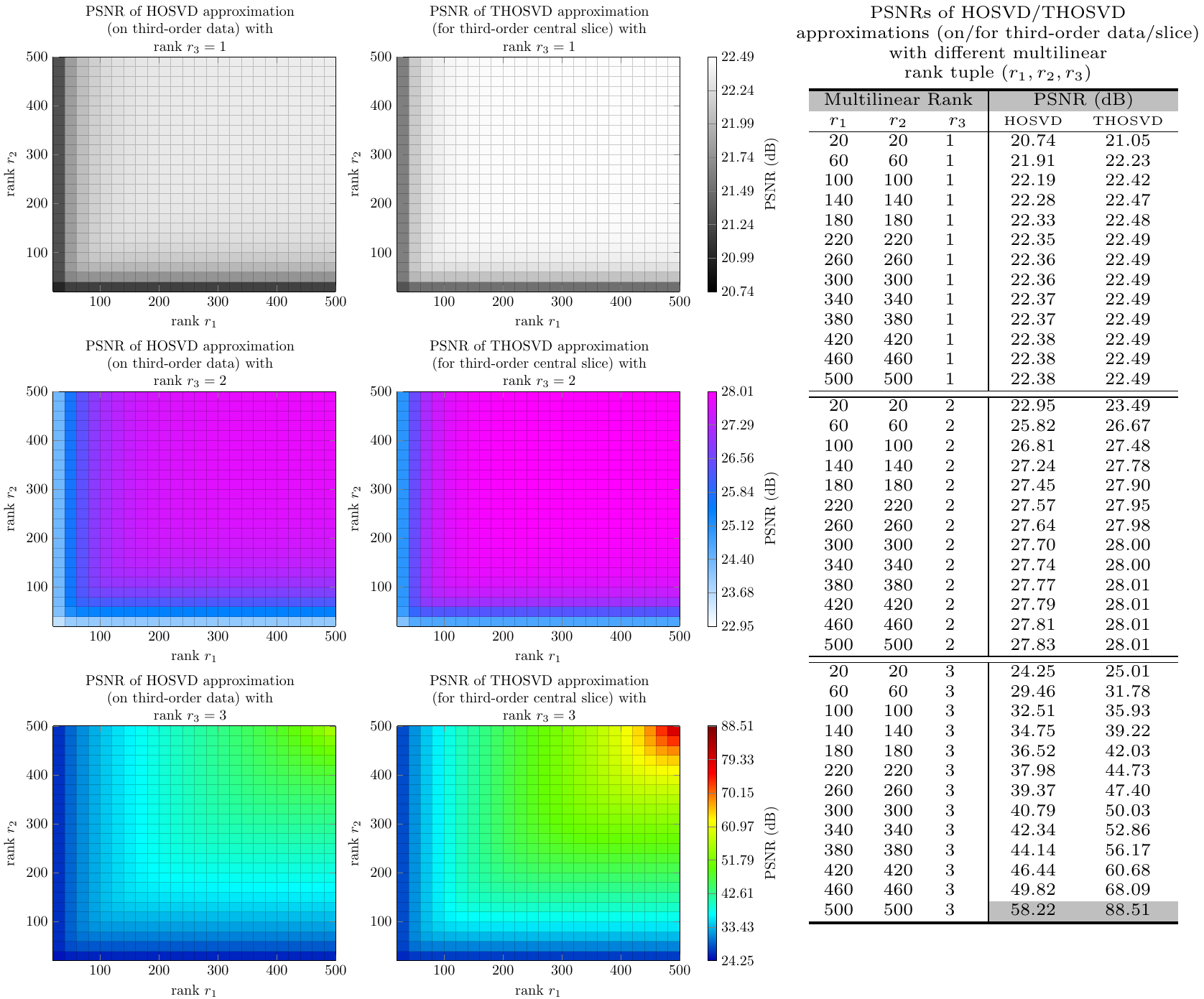}
\end{center}
\caption{A ``vertical'' comparison of THOSVD approximations and HOSVD approximations
with the multilinear rank tuple $(r_1, r_2, r_3)$.~~First column: PSNR maps of HOSVD approximation on the RGB Lena image.
Second column: PSNR maps of THOSVD approximation for the RGB Lena image (i.e., third-order central slice of THOSVD approximation).
Third column: Some quantitative PSNRs of HOSVD/THOSVD approximations with representative multilinear rank tuples.}
\label{fig:lena-THOSV-HOSVD-approximation}
\end{figure*}

\subsubsection{THOSVD versus HOSVD --- A ``Horizontal Comparison''}
Given a fifth-order array $X \in \mathbb{R}^{3\times 3\times 512\times 512\times 5}$
tensorized from the RGB Lena image, which is a third-order array in $\mathbb{R}^{512 \times 512\times  3}$, both THOSVD and HOSVD
can be applied to the same data $X$.

THOSVD takes $X$  as a g-tensor $X_\mathit{GT} \in C^{512\times 512\times 3} \equiv \mathbb{R}^{3 \times 3\times  512\times  512\times  3}$  while HOSVD takes $X$ merely as a canonical fifth-order array in $\mathbb{R}^{3\times 3\times 512\times 512\times 3}$.

Then, given a rank tuple $(r_1, r_2, r_3)$ subject to
$1\le r_1 \le 512$,
$1\le r_2 \le 512$ and
$1\le r_3 \le 3$,
the THOSVD approximation $\hat{X}_\mathit{GT} \in C^{512 \times 512\times  3} $
is computed as in equation (\ref{THOSVD-low-rank-approximation}).

Let the HOSVD of $X \in \mathbb{R}^{3\times 3\times 512\times 512\times 3}$ be
$
X = S ~\times_1~ U_1 ~\times_2~ U_2 ~\times_3~ U_3 ~\times_4~ U_4 ~\times_5~ U_5
$ where $S \in \mathbb{R}^{3\times 3\times 512\times 512\times 3}$ is the core tensor and
$U_1 \in \mathbb{R}^{3\times 3} $,
$U_2 \in \mathbb{R}^{3\times 3} $,
$U_3 \in \mathbb{R}^{512\times 512} $,
$U_4 \in \mathbb{R}^{512\times 512} $,
$U_5 \in \mathbb{R}^{3\times 3} $ are all orthogonal matrices.

Then,
to give a ``horizontal'' comparison with the THOSVD approximation $\hat{X}_\mathit{GT} \in
C^{512\times  512\times 3}$
with a rank tuple $(r_1, r_2, r_3)$, the HOSVD approximation $\hat{X} \in \mathbb{R}^{3\times  3\times  512\times  512\times  3}$ is given by the following multi-mode product
\begin{equation}
\begin{aligned}
\hat{X} = (S)_{:, :, 1:r_1, 1:r_2, 1: r_3} \times_1~ U_1 \times_2~ U_2 \times_3~ (U_3)_{:, 1:r_1} 
\times_4~ (U_4)_{:, 1:r_2}
\times_5~ (U_5)_{:, 1:r_3} \;.
\end{aligned}
\end{equation}

Figure \ref{fig:lena-THOSV-HOSVD-approximation-horizontal-comparison} gives the
``horizontal'' comparison of THOSV approximations and HOSVD approximations on the same array
with different rank tuples $(r_1, r_2, r_3)$.
Albeit somewhat smaller in PSNRs,
the results in Figure \ref{fig:lena-THOSV-HOSVD-approximation-horizontal-comparison}
are similar to the results in Figure \ref{fig:lena-THOSV-HOSVD-approximation} (a ``vertical'' comparison), corroborating the claim that a THOSV approximation outperforms, in terms of PSNR, the corresponding HOSV approximation on the same data.

\begin{figure*}[htb]
\begin{center}
\includegraphics[width=0.9\textwidth]{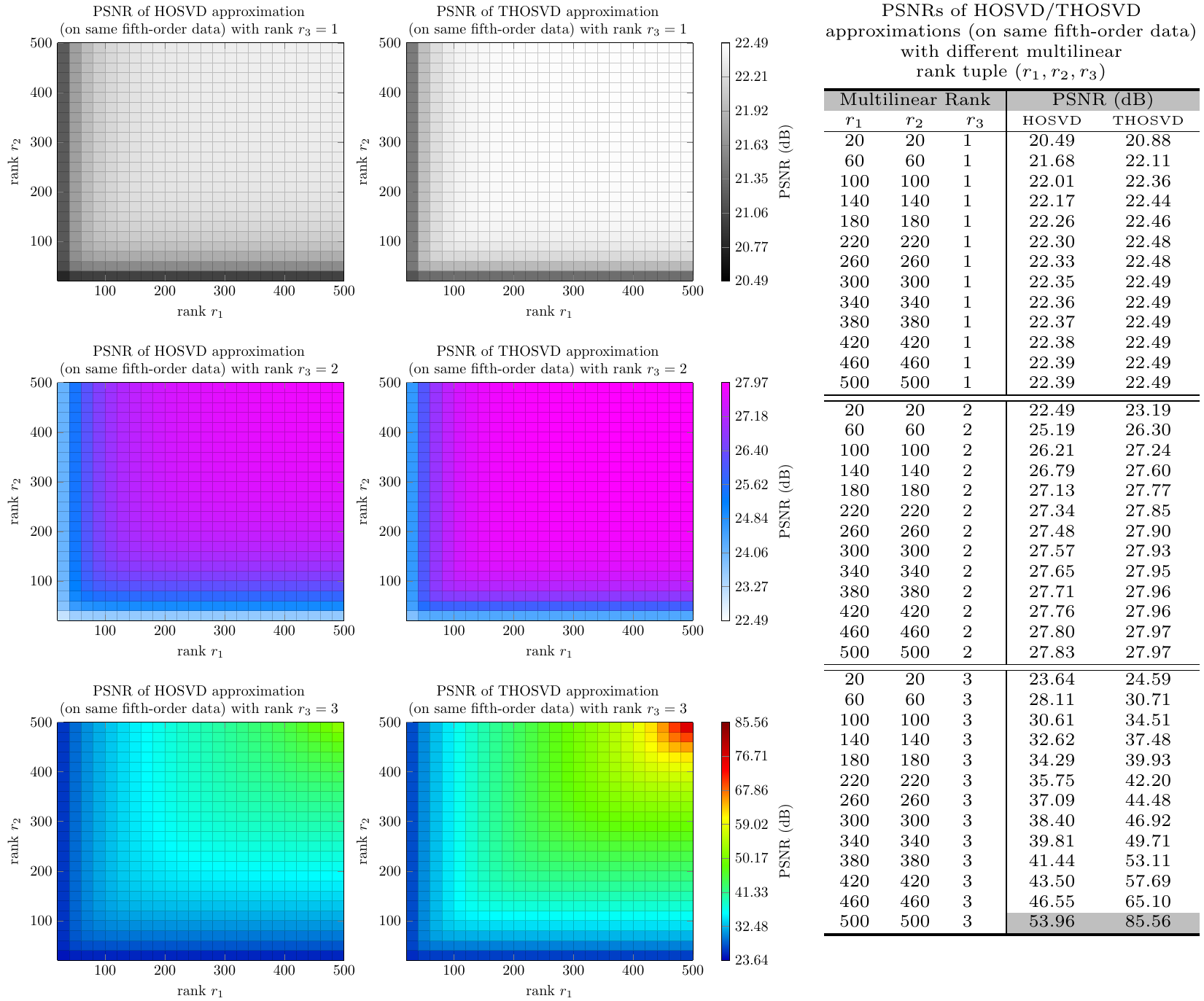}
\end{center}
\caption{
A ``horizontal'' comparison of THOSVD approximations and HOSVD approximations
with multilinear rank tuple $(r_1, r_2, r_3)$.~~
First column: PSNR maps of HOSVD approximation.
Second column: PSNR maps of THOSVD approximation. Third column: Some PSNRs of
HOSVD/THOSVD approximations on same fifth-order data with representative multilinear rank tuples $(r_1, r_2, r_3)$.
}
\label{fig:lena-THOSV-HOSVD-approximation-horizontal-comparison}
\vspace{1em}
\end{figure*}

\subsection{Reconstruction}
\label{section:Reconstruction}

The qualities of the low-rank reconstructions produced by TPCA and PCA and by T2DPCA and 2DPCA, as described  by the equations (\ref{equation:reconstruction-of-TPCA})  and
(\ref{equation:T2DPCA-reconstruction}), are compared.

The effectiveness of
PCA, 2DPCA, TPCA and T2DPCA for reconstruction is assessed
using the ORL dataset.
The data set contains $400$ face images in $40$ classes, i.e., $10$ images/class $\times$ $40$ classes. Each image has
$112\times 92$ pixels\footnote{\url{https://www.cl.cam.ac.uk/research/dtg/attarchive/facedatabase.html}}.
The first $200$ images ($5$ images/class $\times$ $40$ classes) are used as the observed images and the remaining $200$ images are the query images.

For the experiments with TPCA/T2DPCA,
all ORL images are tensorized to t-images in $R^{112\times 92}$, namely,
order-four arrays in $\mathbb{R}^{3\times 3\times 112\times 92 }$.
Eigendecompositions and t-eigendecompositions
are computed on the observed images and t-images, respectively. Reconstructions are computed for the
query images and t-images respectively.
The number of PSNRs for the reconstructed images and t-images is $200$. It is convenient to use the average of the PSNRs (denoted by $A$), the standard deviation of PSNRs (denoted by $S$), and the ratio, $A/S$. A larger value of $A$ with a smaller value of $S$, indicates a better quality of reconstruction.

\subsubsection{TPCA versus PCA --- A ``Vertical'' Comparison}

To make the TPCA and PCA reconstructions computationally  tractable, each image is resized to $56 \times 46$ pixels by bi-cubic interpolation. The resized images are also tensorized to t-images, i.e., order-four arrays in $\mathbb{R}^{3\times 3\times 56\times 46}$.
The obtained images and t-images are then transformed to vectors and t-vectors, respectively, by stacking their columns. The central slices of the TPCA reconstructions are compared with the PCA reconstructions.

Figure \ref{figure:recontruction-PCA-TPCA} shows graphs and some tabulated values of $A$, $S$ and $A/S$ for a number of eigen-vectors and eigen-t-vectors. Note that $K$ linearly independent observed vectors or t-vectors yield at most $(K-1)$ eigen-vectors or eigen-t-vectors. Thus, the maximum number of eigen-vectors and eigen-t-vectors in Figure \ref{figure:recontruction-PCA-TPCA} is $199$ ($K = 200$).

The average PSNR for TPCA is consistently higher than the average PSNR for PCA.
The PSNR standard deviation for TPCA is slightly larger than the PSNR standard deviation for PCA, but the ratio $A/S$ for TPCA is generally smaller than the ratio $A/S$ for PCA. This indicates that TPCA outperforms PCA in terms of reconstruction quality.

\begin{figure*}[htb]
\begin{center}
\includegraphics[width=0.85\textwidth]{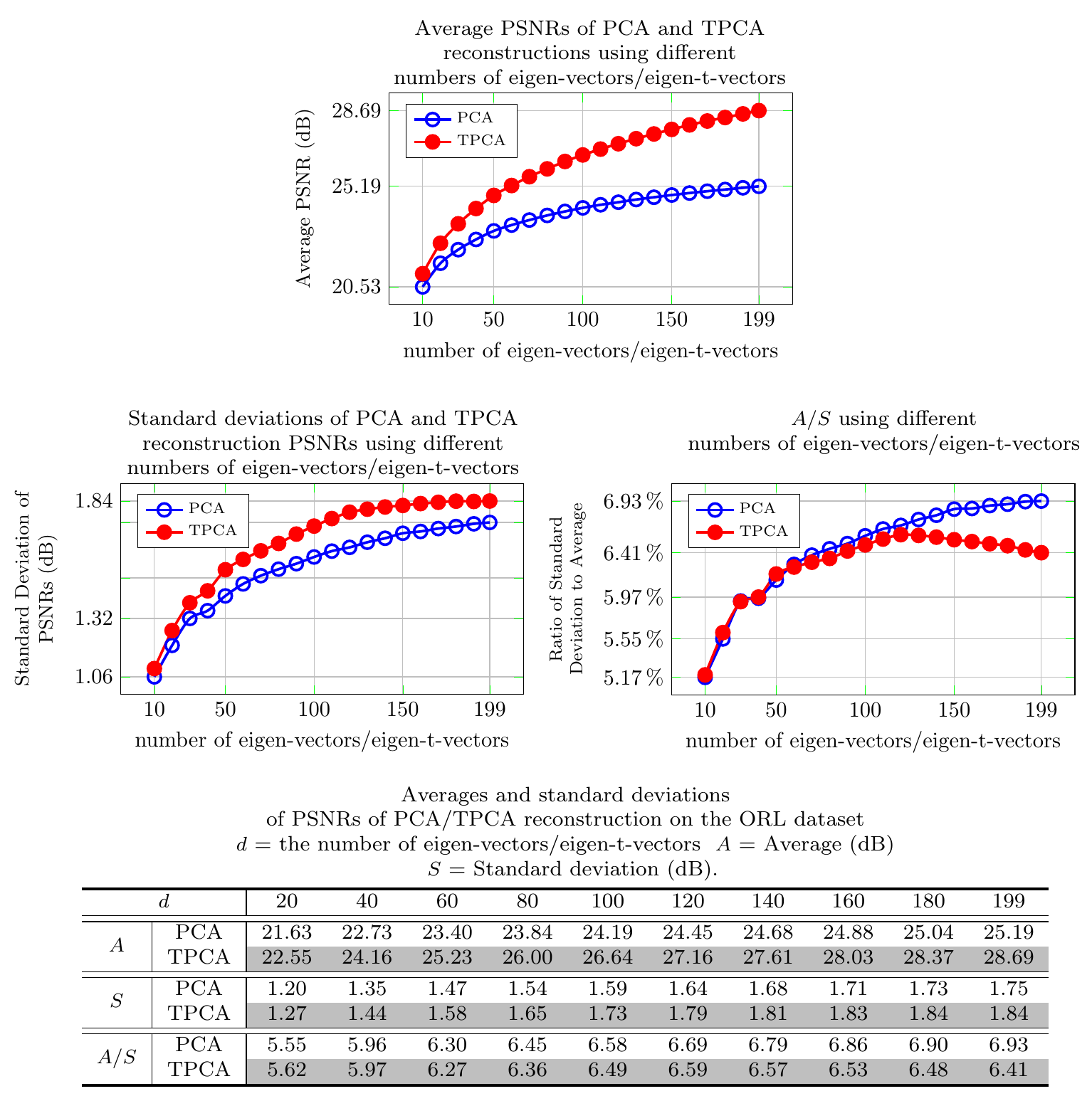}
\end{center}
\caption{A  ``vertical'' comparison of PSNR averages and standard deviations for PCA and TPCA reconstructions.}
\label{figure:recontruction-PCA-TPCA}
\end{figure*}

\subsubsection{T2DPCA versus 2DPCA --- A ``Vertical'' Comparison}

The same observed samples from the ORL dataset
(the first $200$ images, $5$ images/class $\times$ $40$ classes)
and query samples
(the remaining $200$ images)  are used to compare the reconstruction performances of T2DPCA and 2DPCA. The central slices of the T2DPCA are compared with the 2DPCA reconstructions.

Figure \ref{fig:average-deviation-2DPCA-T2DPCA} shows the reconstruction curves and some tabulated values yielded by T2PCA and 2DPCA as functions of the number $d$ of eigenvectors or eigen-t-vectors. The average PSNR obtained by T2DPCA is consistently higher than the average PSNR obtained by 2DPCA.
When the parameter $d$ equals $111$, the gap between the two average PSNRs is
$31.98$ dBs. Furthermore, the PSNR standard deviation for T2DPCA is also generally smaller than the PSNR standard deviation for 2DPCA. In terms of reconstruction quality, T2DPCA outperforms 2DPCA.

\begin{figure*}[htb]
\begin{center}
\includegraphics[width=0.9\textwidth]{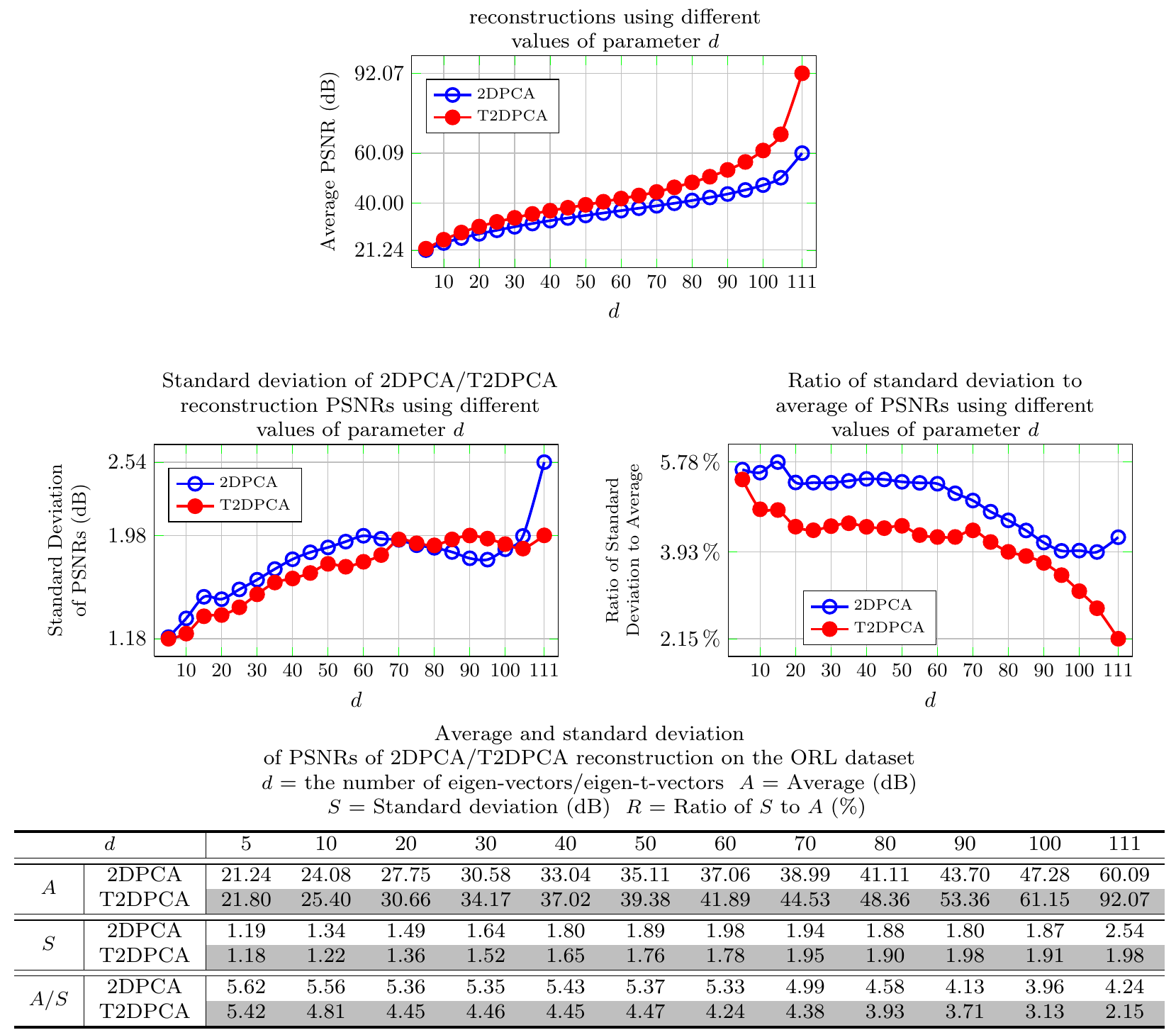}
\end{center}
\caption{A ``vertical'' comparison of PSNR averages and standard deviations for the 2DPCA and T2DPCA reconstructions. }
\label{fig:average-deviation-2DPCA-T2DPCA}
\end{figure*}

\subsection{Classification}
\label{section:Classification}

TGCA and GCA are applied to the classification of the pixel values in hyperspectral images. Hyperspectral images have hundreds of spectral bands, in contrast with RGB images which have only three spectral bands. The multiple spectral bands and high resolution make hyperspectral imagery essential in remote sensing, target analysis, classification and identification \cite{Liu2017Sparse,He2016Tensor-0008,Zhang2016Spectral,Fu2016Hyperspectral,Wei2017Spectral,Ma2016Spectral,Zhang2013Tensor011}.
Two publicly available data sets are used to evaluate the effectiveness of TGCA and GCA for supervised classification.

\subsubsection{Datasets}

The first hyperspectral image dataset is the Indian Pines cube (Indian cube for short),
which consists of $145 \times 145$ hyperspectral pixels (hyperpixels for short)
and has $220$ spectral bands, yielding an array of order-three in
$\mathbb{R}^{145 \times 145 \times 220}$.
The Indian cube comes with ground-truth labels for $16$ classes \cite{Hyperspectrallink}.
The second hyperspectral image dataset is the Pavia University cube (Pavia cube for short), which consists of $610 \times 340$ hyperpixels with $103$ spectral bands, yielding an array of order three in $\mathbb{R}^{610 \times 340 \times 103}$. The ground-truth contains $9$ classes \cite{Hyperspectrallink}.

\subsubsection{Tensorization}

Given a hyperspectral cube, let $D_1$ be he number of rows, $D_2$ the number of columns and
$D$ the number of spectral bands. A hyperpixel is represented by a vector in $\mathbb{R}^{D}$. Each pixel is tensorized by its  $3\times 3$ neighborhood. The tensorized hyperspectral cube is represented by an array in
$\mathbb{R}^{3\times 3\times D_1\times D_2\times D}$. Each tensorized hyperpixel, called t-hyperpixel in this paper, is represented by a t-vector in $R^{D}$, i.e., an array in $\mathbb{R}^{ 3\times 3\times D}$.

Figure \ref{fig:data-tensorization} shows the tensorization of a canonical vector extracted from a hyperspectral cube. The tensorization of all vectors yields a tensorized
hyperspectral cube in $\mathbb{R}^{3\times 3\times D_1\times D_2\times D}$.

\begin{figure}[b]
\centering
    \includegraphics[width=0.46\textwidth]{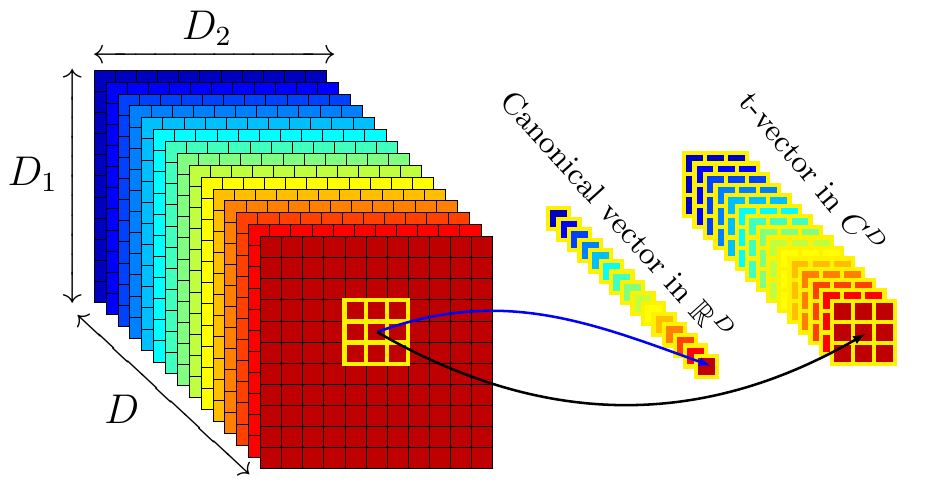}
\vspace{1em}
\caption{Tensorization of a canonical vector extracted from a hyperspectral cube}
\label{fig:data-tensorization}
\end{figure}

\subsubsection{Input Matrices and T-matrices}

To classify a query hyperpixel, it is necessary to extract features from the hyperpixel.
A t-hyperpixel in TGCA is represented by a set of t-vectors in the $5\times 5$ neighborhood of the
t-hyperpixel. These t-vectors are used to construct a t-matrix. A similar construction is used for GCA.

In GCA for example, let the vectors in the $5\times 5$ neighborhood of a hyperpixel be $X_{\mathit{vec}, 1}, \cdots, X_{\mathit{vec}, 25}$. The ordering of the vectors should be the same for all hyperpixels. The raw matrix
$X_\mathit{mat}$ representing the hyperpixel is
given by marshalling these vectors as the columns of $X_\mathit{mat}$, namely $X_\mathit{mat} \doteq [X_{\mathit{vec}, 1}, \cdots, X_{\mathit{vec}, 25}] \in \mathbb{R}^{D\times 25}$.
The associated t-matrix $X_\mathit{TM}  \in C^{D\times 25}$ in TGCA is obtained by marshalling the associated $25$ t-vectors.

After obtaining each matrix and t-matrix, the columns are orthogonalized. The resulting matrices and t-matrices are input samples for GCA and TGCA respectively.

\subsubsection{Classification}

To evaluate GCA, TGCA and the competing methods, the overall accuracies (OA) and
the Cohen's $\kappa$ indices of the supervised classification of hyperpixels (i.e.,
prediction of class labels of hyperpixels) are used. The overall accuracies and $\kappa$ indices are
obtained for different component analysers
and classifiers. Higher values of OA or $\kappa$  indicate a  higher component analyzer performance
\cite{Fitzgerald1994Assessing-kappa-definition}. Let $K$ be the number of query samples, let $K'$ be the number of correctly classified samples. The overall accuracy  is simply defined by
$\mathit{OA} = K^{'} / K $.
The $\kappa$ index is defined by \cite{Cohen1960ACoefficient}
\begin{equation}
\kappa = \frac{K \cdot K^\mathit{'} - \sum\nolimits_{j=1}^{N^\mathit{class}} a_{j}b_{j}}{{K^{2} - \sum\nolimits_{j=1}^{N^\mathit{class}} a_{j}b_{j}}}
\end{equation}
where $N^\mathit{class}$ is the number of classes,
$a_{j}$ is the number of samples belonging to the $j$-th  class and
$b_{j}$ is the number of samples classified to the $j$-th class.

Two classical component analyzers, namely PCA and LDA, and four state-of-the-art component analyzers, namely TDLA \cite{Zhang2013Tensor011}, LTDA \cite{Zhong2015Discriminant017},
GCA \cite{Harandi2015Extrinsic} and TPCA (ours)
are evaluated against
TGCA.
As an evaluation baseline, the results
obtained with the original raw canonical vectors for hyperpixels are given. These raw vectors are
denoted as the ``original'' (ORI for short) vectors.
Three vector-oriented  classifiers, NN (Nearest Neighbor), SVM (Support Vector Machine), and RF (Random Forest), are employed to evaluate the
effectiveness of the features extracted by  these component analyzers.

In the experiments, the background hyperpixels are excluded, because they do not have labels in the
ground-truth. A total of $10\%$ of the foreground hyperpixels
are randomly and uniformly chosen without replacement as the observed
samples (i.e., samples whose class labels are known in advance).
The rest of the foreground hyperpixels
are chosen as the query samples, that is samples with the class labels to be determined.

In order to use
the vector-oriented classifiers NN, SVM and RF,
the t-vector results, generated by TGCA or TPCA, are transformed
by pooling them to yield canonical vectors. For TGCA, the canonical vectors obtained by pooling are referred to as TGCA-I features and the t-vectors without pooling are referred to as the TGCA-II features.

To assess the effectiveness of the TGCA-II features, a generalized classifier which deals with t-vectors is needed. It is possible to generalize many canonical classifiers from vector-oriented to t-vector-oriented, however a comprehensive  discussion of these generalizations is outside the scope of this paper. Nevertheless, it is very straightforward to generalize NN. The $d$-dimensional t-vectors are not only elements of the module $C^{d}$, but also the elements in the
vector space $\mathbb{C}^{3\times 3 \times d}$.
This enables the use of the canonical Frobenius norm
to measure the distance between two t-vectors,
as the elements in $\mathbb{C}^{3\times 3 \times d}$.
The canonical Frobenius norm should not be confused with the generalized Frobenius norm defined in equation (\ref{A7}).

Figure \ref{figure:barplot-classification-accuracies} gives the highest classification accuracies obtained by each pair of component analyser and classifier on the two hyperspectral cubes.
The highest accuracies are obtained by traversing the set of feature dimensions $d \in \{5, 10,\cdots, D_m\}$ where $D_\mathit{m}$ is the maximum dimension valid for the associated component analzser.
Figure \ref{figure:barplot-classification-accuracies},
shows that
the results obtained by the algorithms TPCA, TGCA-I and TGCA-II, are consistently
better than those obtained by their canonical counterparts.
Even working with a relatively weak classifier NN, TGCA achieves the highest accuracies and highest $\kappa$ indices in the experiments. Further results are shown in Figures \ref{fig:figure005-result-indian} and \ref{fig:figure005-result-Pavia}.
It is clear that the pair TGCA and NN
yield the best results, outperforming any other pair of analyzer and classifier.

\begin{figure*}[htb]

\centering
\includegraphics[width=0.9\textwidth]{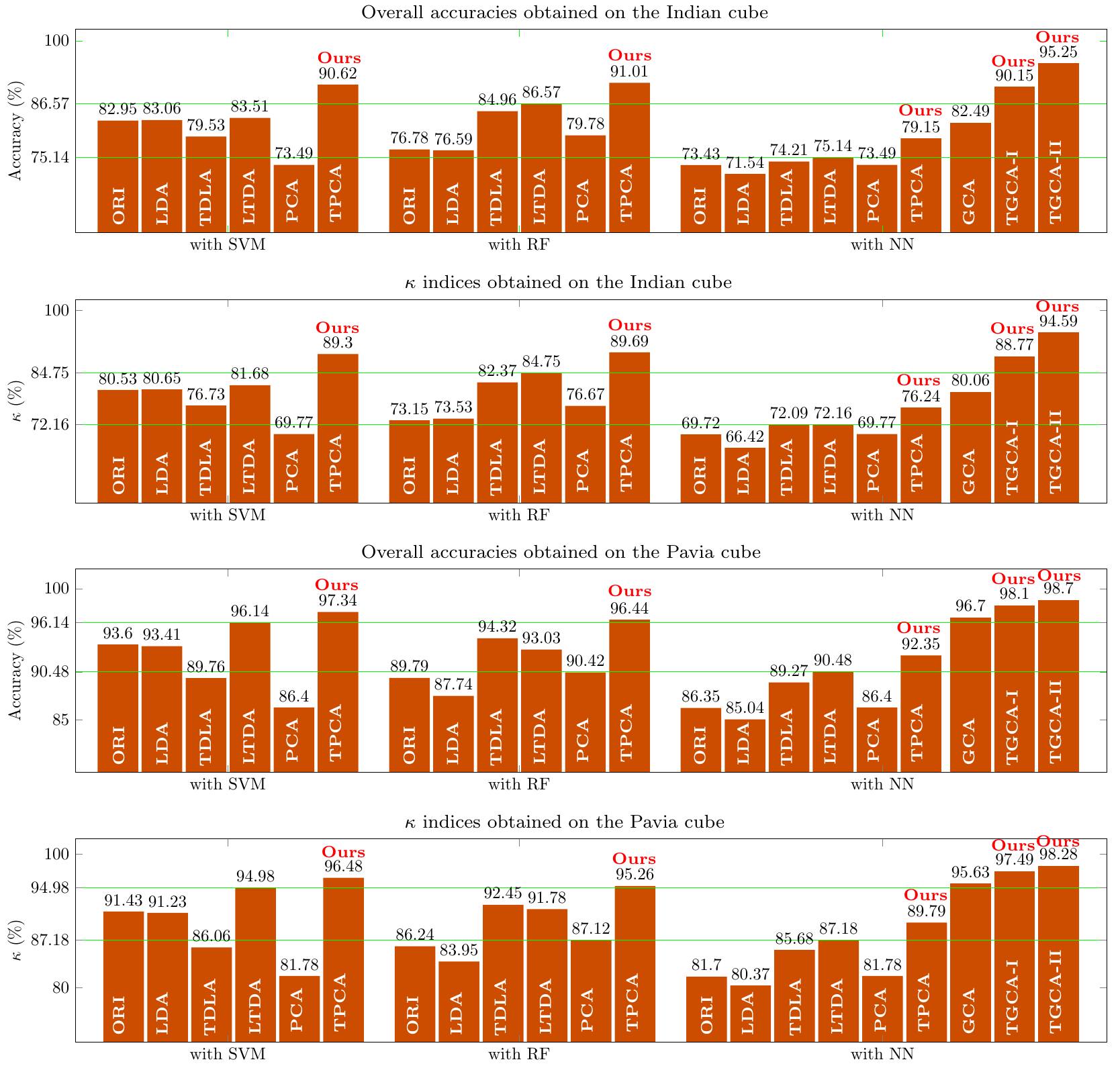}

\caption{Classification accuracies obtained on two hyperspectral cubes}
\label{figure:barplot-classification-accuracies}
\end{figure*}

\begin{figure*}[htb]

\centering
\includegraphics[width=0.9\textwidth]{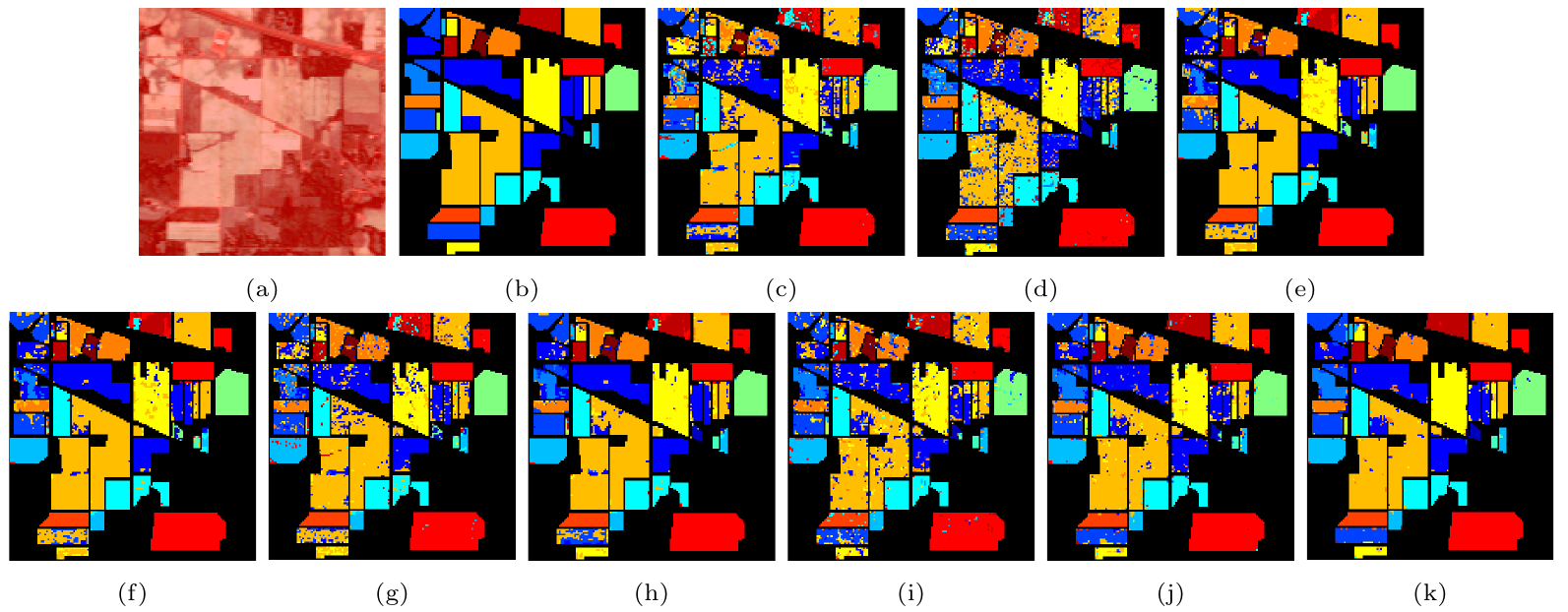}
\caption{Some visual results obtained on the Indian cube.
  ~(a) Pseudo-coloured 2D scene of Indian Pines
  ~(b) class ground-truth of hyperpixels
  ~(c) ORI with RF
  ~(d) LDA with RF
  ~(e) TDLA with RF
  ~(f) LTDA with RF
  ~(g) PCA with RF
  ~(h) TPCA with RF
  ~(i) GCA with NN
  ~(j) TGCA-I with NN
  ~(k) TGCA-II with NN}
  \label{fig:figure005-result-indian}

\end{figure*}

\begin{figure*}[htb]

\centering
\includegraphics[width=0.9\textwidth]{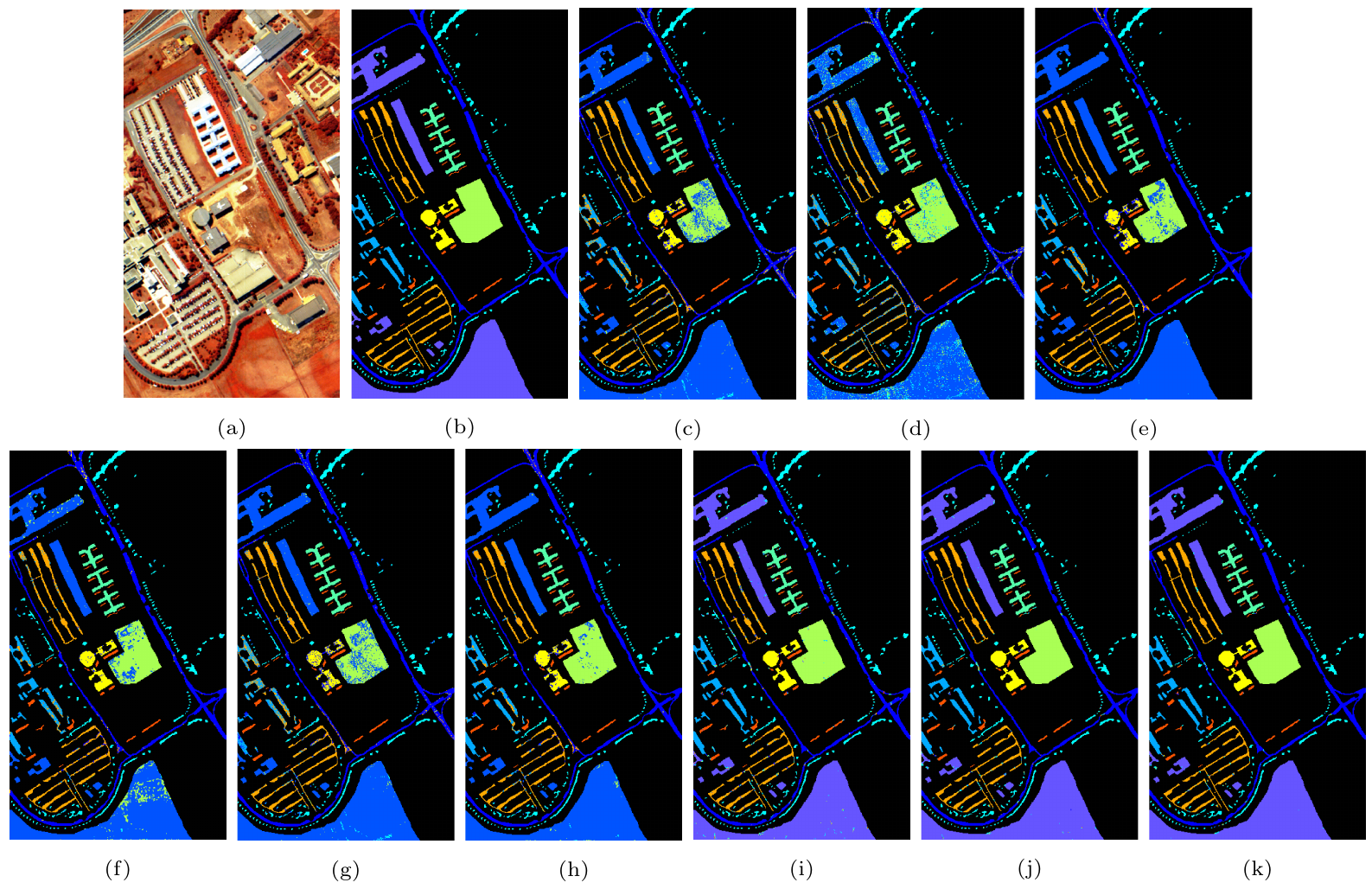}
\caption{Some visual results obtained on the Pavia cube.
  (a) Pseudo-colored 2D scene of the Pavia University
  ~(b) class ground-truth of hyperpixels
  ~(c) ORI with  RF
  ~(d) LDA with RF
  ~(e) TDLA with RF
  ~(f) LTDA with RF
  ~(g) PCA  with RF
  ~(h) TPCA  with RF
  ~(i) GCA  with NN
  ~(j) TGCA-I  with NN
  ~(k) TGCA-II  with NN}
  \label{fig:figure005-result-Pavia}
\end{figure*}

\subsubsection{TGCA versus GCA}

It is noted that the maximum dimension of the TGCA and GCA features is
equal to the number of observed training samples, and therefore is much higher than
the original dimension, which is equal to
the number of spectral bands.
Thus, taking the original dimension as the baseline,
one can employ TGCA or GCA either for dimension reduction or dimension increase.
When the so-called ``curse of dimension'' is the concern, one can
discard the insignificant entries of the TGCA and GCA features. When the accuracy is the primary concern, one can use higher dimensional features.

\begin{figure*}[htb]
\centering
\includegraphics[width=0.9\textwidth]{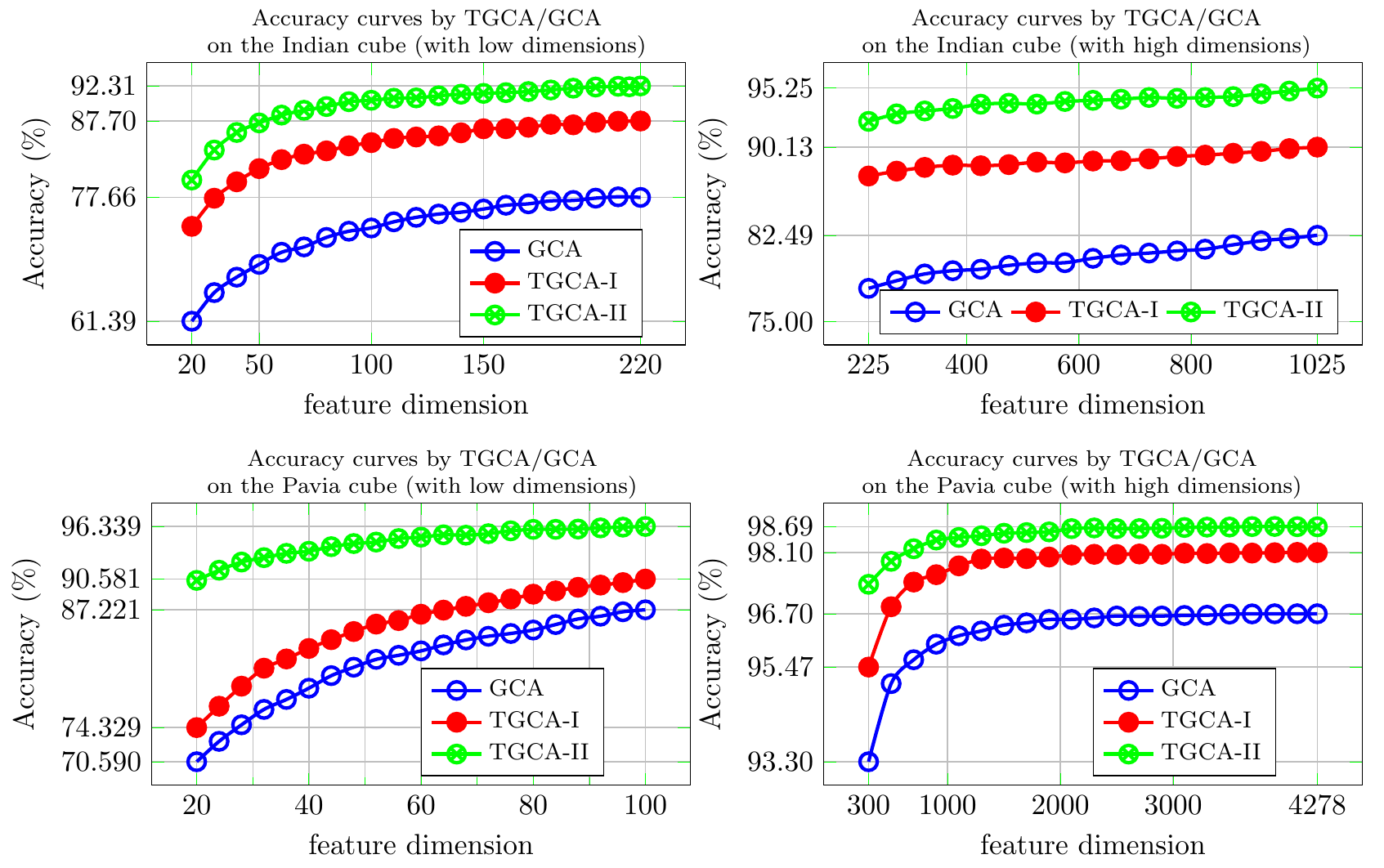}
\caption{Accuracy curves obtained by  TGCA/GCA (with NN) on the Indian/Pavia cube}
\label{fig:accuracy-comparison-GCA-TGCA}
\end{figure*}

The performances of TGCA and GCA for varying feature dimension are compared using accuracy curves
generated by TGCA (ie., TGCA-I and TGCA-II) and GCA, as shown in Figure \ref{fig:accuracy-comparison-GCA-TGCA}.
The results are obtained for low feature dimensions and for high feature dimensions.
It is clear that the classification accuracies obtained using TGCA and TGCAII are consistently higher
than the accuracies obtained using GCA.

\subsubsection{TPCA versus PCA}

The classification accuracies of TPCA and PCA are compared, although the highest classification accuracies are not obtained from TPCA or PCA. The classification accuracy curves obtained by TPCA and PCA  (with classifiers NN, SVM and RF) are given in
Figure \ref{fig:accuracy-comparison-TPCA-PCA-classifiers}. It is clear that, no matter which classifier and feature dimension are chosen, the accuracy using TPCA
is consistently higher than the accuracy using PCA.\footnote{To use the same classifiers, pooling is used to transform the t-vectors by TPCA to canonical vectors.}

\begin{figure*}[htb]
\centering
\includegraphics[width=0.9\textwidth]{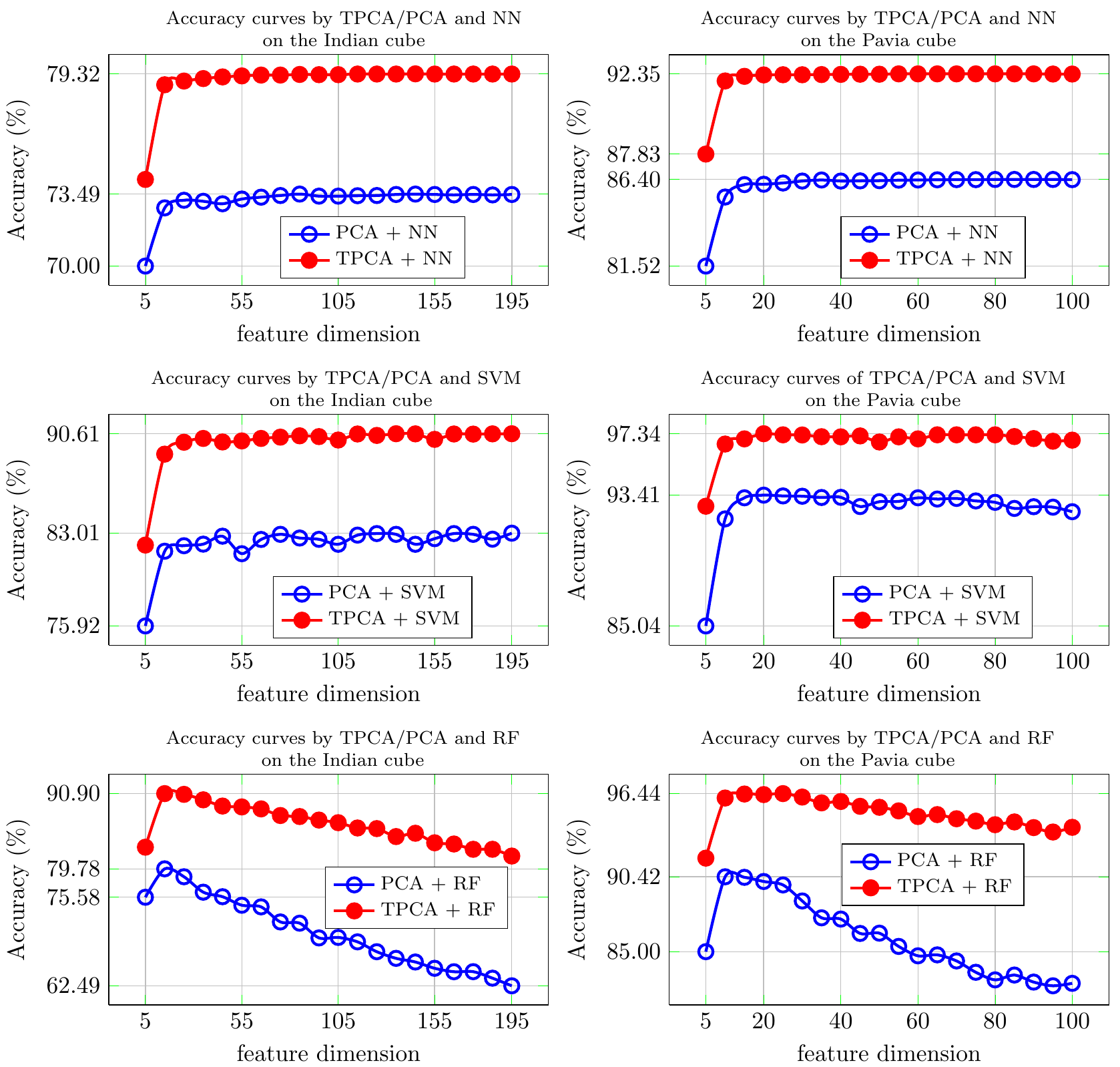}

\caption{Accuracy curves obtained by  TPCA/PCA and different classifiers on the Indian/Pavia cube. First column: results on the Indian cube. Second column: results on the Pavia cube. }
\label{fig:accuracy-comparison-TPCA-PCA-classifiers}
\end{figure*}

\subsection{Computational Cost}

The run times of t-matrix manipulations with different t-scalar sizes $I_1 \times I_2$ are given in Figure \ref{figure:run-time}. The size of t-scalars ranges from $1 \le I_1, I_2\le  32$.
The evaluated t-matrix manipulations include addition, conjugate transposition, multiplication and TSVD. The run time is evaluated using MATLAB R2018B on a notebook PC with Intel i7-4700MQ CPU at 2.40GHz and 16G GB memory.

Each time point in the figure is obtained by averaging $100$ manipulations on random t-matrices in $\mathbb{R}^{I_1\times I_2\times 64 \times 64}$.
Each t-matrix with $(I_1, I_2) \neq (1, 1)$  is transformed to the Fourier domain and manipulated via its $I_1\cdot I_2$ slices. The results are transferred back to the original domain by the inverse Fourier transform.
Note that when $(I_1, I_2) = (1, 1)$, a t-matrix manipulation is reduced to canonical matrix manipulation. The reported run time of a canonical matrix manipulation does not includes the time spent on the Fourier transform and its inverse transform.

From Figure \ref{figure:run-time}, it can be seen that the run time is essentially an increasing linear function of the number of slices, i.e., $I_1\cdot I_2$ .

\begin{figure*}[htb]
\centering
\includegraphics[width=0.85\textwidth]{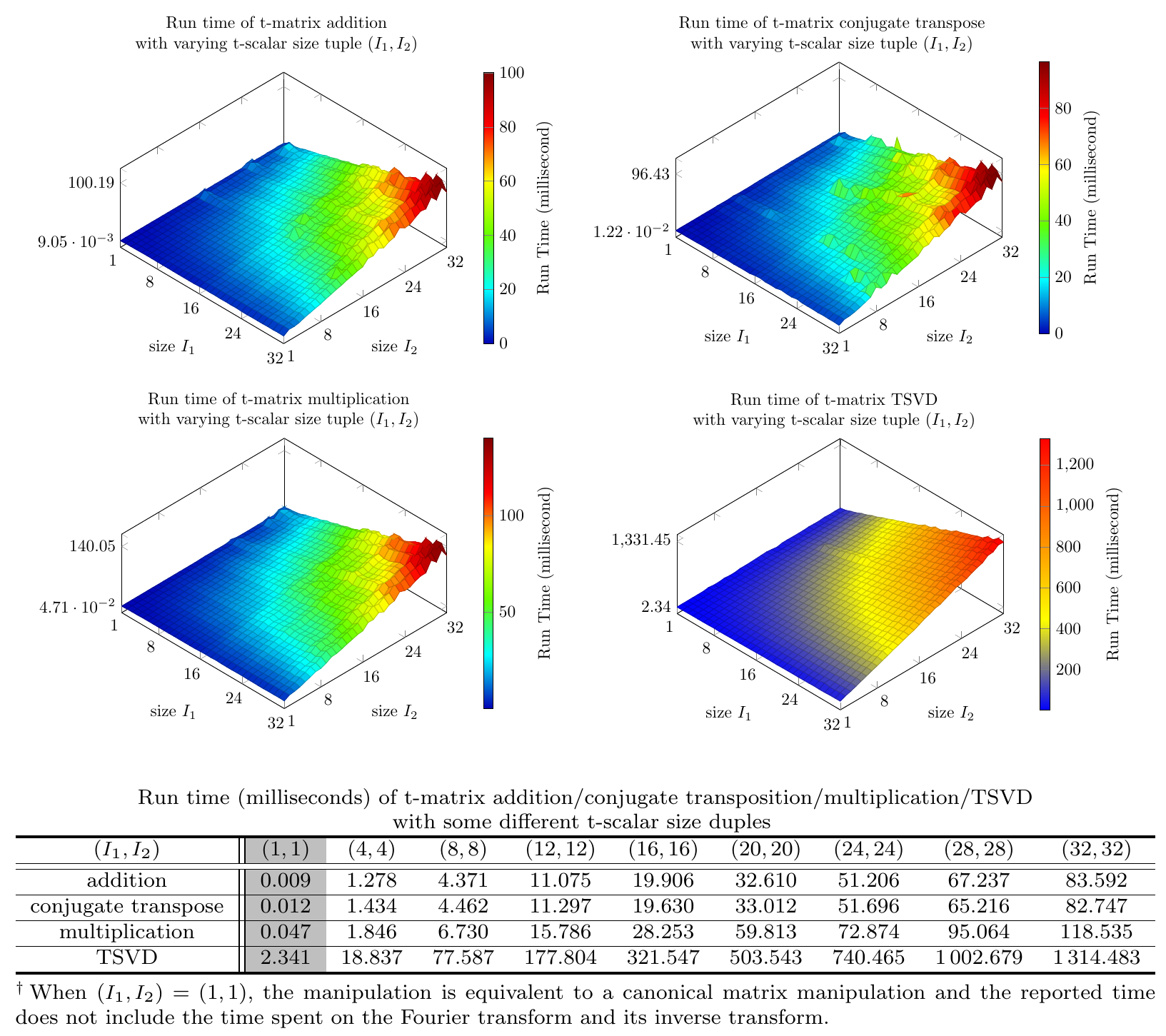}
\caption{Run time of some t-matrix manipulations with different t-scalar sizes}
\label{figure:run-time}
\end{figure*}

\section{Conclusion}
\label{section:conclusions}

An algebraic framework of tensorial matrices is proposed for generalized visual information analysis.
The algebraic framework generalizes the canonical matrix algebra,
combining the ``multi-way'' merits of high-order arrays
and the ``two-way'' intuition of matrices. In the algebraic framework, scalars are extended to t-scalars,
which are implemented as high-order numerical arrays of a fixed-size.
With appropriate operations,
the t-scalars are trinitarian in the following sense. (1) T-scalars are generalized complex numbers.
(2) T-scalars are elements of an algebraic ring. (3) T-scalars are elements of a linear space.

Tensorial matrices, called t-matrices, are constructed with t-scalar elements.
The resulting t-matrix algebra is backward compatible with the canonical matrix algebra.
Using this t-algebra framework, it is possible to generalize many canonical matrix and vector constructions and algorithms.

To demonstrate the ``multi-way'' merits and ``two-way'' matrix intuition of the proposed tensorial algebra and its applications to generalized visual information analysis, the canonical matrix algorithms SVD, HOSVD, PCA, 2DPCA and GCA are generalized.
Experiments with low-rank approximation, reconstruction, and supervised classification
show that the generalized algorithms compare favorably with their canonical counterparts on visual information analysis.

\section*{acknowledgements}
Liang Liao would like to thank professor Pinzhi Fan (Southwestern Jiaotong University, China) for his support and some insightful suggestions to this work. Liang Liao also would like to thank Yuemei Ren, Chengkai Yang, Haichang Ye,
Jie Yang and Xuechun Zhang for their supports to some early stage experiments of this work.

All prospective supports and collaborations to this research are welcome. Contact email: liaolangis@126.com or liaoliang2018@gmail.com.

\section*{Open Source}

A MATLAB repository on the t-algebra, t-vectors, and t-matrices is open-sourced at 
www.github/liaoliang2020/talgebra. Interested readers are referred to this URL for more details.

\clearpage
\section*{Appendix I}

Before giving a proof of the equivalence of equations (\ref{equation:TSVD-low-rank-approximation111}) and (\ref{equation:generalized-Eckart-theorem}), namely, the generalized
Eckart-Young-Mirsky theorem, some
notations need to be defined.

First, $\operatorname{rank}(\cdot)$ denotes the rank of a t-matrix, which generalizes the rank of a canonical matrix and is defined as follows.

\vspace{0.5em}
\noindent
\textit{{Definition I, rank of a t-matrix}}. ~Given a t-matrix,
the rank $Y_{T} \doteq \operatorname{rank}(X_\mathit{TM})$ is a nonnegative t-scalar  such that
\begin{equation}
F(Y_{T})_i = \operatorname{rank}(F(X_\mathit{TM})(i)) \ge 0 \;,\;\;  1\le i \le I\;\;.
\label{equation:generalized-rank}
\end{equation}
where
$F(X_\mathit{TM})(i)$ denotes the $i$-th slice of the Fourier transform $F(X_\mathit{TM})$.

\vspace{0.5em}
\noindent
\textit{{Definition II, partial ordering of nonnegative t-scalars}}.
~Given two nonnegative
t-scalars $X_{T}$
and $Y_{T}$, the notation $X_T \le Y_T$ is equivalent to the following condition
\begin{equation}
F(X_T)_i \le F(Y_T)_i \;,\;\; 1 \le i\le I \;\;.
\label{equation:partial-ordering}
\end{equation}

\vspace{0.5em}
\noindent
\textit{Definition III, minimization of nonnegative t-scalar variable}.
~For a nonnegative t-scalar variable $X_T$
varying in a subset of $ S^\mathit{nonneg}$
, $Y_T \doteq \min(X_T)$ is the  nonnegative t-scalar
infimum of the subset, satisfying the following condition.
\begin{equation}
F(Y_T)_{ i} = \min\left(F{(X_T)}_{i} \right) \ge 0\;,\; 1 \le i\le I
\end{equation}
where $F(Y_T)$ and $F(X_T)$ respectively denote the Fourier transforms of $Y_T$ and
$X_T$.

Given two nonnegative t-scalars $X_{T}$ and $Y_{T}$, let $M_{T}$ be the nonnegative t-scalar defined by $M_{T} = \min(X_{T}, Y_{T})$, namely
\begin{equation}
F({M}_T)_i = \min(F({X}_T)_i, F({Y}_T)_i) \ge 0\;\;,\; i \le i \le I .
\end{equation}

The above definitions are not casual ones.
Following the above definitions, it is not difficult to verify that many
generalized rank
properties hold in the analogous form of their canonical counterparts.

For examples, given any t-matrices
$
X_\mathit{TM} \in C^{D_1\times D_2}$ and
$
Y_\mathit{TM} \in C^{D_2\times D_3}$, the following inequalities hold.

\begin{equation}
\begin{aligned}
Z_T \le \operatorname{rank}(X_\mathit{TM}) \le \min(D_1, D_2) \cdot E_T\;\;.
\end{aligned}
\end{equation}

\begin{equation}
\begin{aligned}
Z_T &\le \operatorname{rank}(X_\mathit{TM} + Y_\mathit{TM}) \\
&\le \operatorname{rank} (X_\mathit{TM})
+ \operatorname{rank} (Y_\mathit{TM}) \;\;.
\end{aligned}
\end{equation}

\begin{equation}
\begin{aligned}
&~~~~\operatorname{rank}(X_\mathit{TM}) + \operatorname{rank}(Y_\mathit{TM}) -D_2 \cdot E_T \\
&\le \operatorname{rank} (X_\mathit{TM} \circ Y_\mathit{TM}) \\
&\le
\min\big( \operatorname{rank}(X_\mathit{TM}), \operatorname{rank}(Y_\mathit{TM}) \big) \\
\end{aligned}
\end{equation}

Since a t-scalar is a t-matrix of one row and one column,
the rank of a t-scalar can be obtained.

Given any t-scalar  $X_T$, let $G_T \doteq \operatorname{rank}(X_T)$
be the rank of $X_T$.
Then, following equation (\ref{equation:generalized-rank}), it is not difficult to prove
that the $i$-th entry of the Fourier transform $F(G_T)$ is given as follows.
\begin{equation}
F(G_T)_i = \left\{
\begin{aligned}
1,  & \text{~~~if~} X_{T, i} \neq 0\\
0,  & \text{~~~otherwise~}
\end{aligned}
\right .
\;\;\;,\;\;  1\le i \le I \;.
\end{equation}

Following the partial ordering given as in
(\ref{equation:partial-ordering})
and equation (\ref{equation:generalized-rank}), it is not difficult to prove that the following propositions hold.
\begin{equation}
\begin{aligned}
&Z_T \le \operatorname{rank}(X_T) \le E_T\;,  \text{for all t-scalars~}  \\
&Z_T = \operatorname{rank}(X_T) \text{~iff~} X_T = Z_T \\
&E_T = \operatorname{rank}(X_T) \text{~iff~} X_T \text{~is invertible}.
\end{aligned}
\label{equation:propositions}
\end{equation}

It follows from (\ref{equation:propositions}) that
$Z_T < \operatorname{rank}(X_T) < E_T $ iff
the t-scalar $X_T$ is non-zero and non-invertible.\footnote{ The partial order ``$<$'' is defined between nonnegative
t-scalars.
The inequality
$Z_T < \operatorname{rank}(X_T) <E_T$ means
$Z_T \le \operatorname{rank}(X_T) \le E_T$ and
$\operatorname{rank}(X_T) \neq Z_T$ and
$\operatorname{rank}(X_T) \neq  E_T$.   }

\vspace{1em}
\noindent
\textit{Generalized rank from a TSVD perspective}.
Given any t-matrix $X_\mathit{TM} = U_\mathit{TM} \circ S_\mathit{TM} \circ V_\mathit{TM}^{\myhtrans} $ where
$S_\mathit{TM} \doteq
\operatorname{diag}(\lambda_{T, 1}, \cdots, \lambda_{T, k}, \cdots \lambda_{T, Q} )$ and $\lambda_{T, k} \in C$ is a t-scalar for all $k$,
then the following equation holds and generalizes its canonical counterpart.
\begin{equation}
\begin{matrix}
Z_T \le \operatorname{rank}(X_\mathit{TM}) \equiv \sum\nolimits_{k=1}^{Q} \operatorname{rank} (\lambda_{T, k}) \le  Q \cdot E_T  \;\;.
\end{matrix}
\end{equation}

Let the approximation be $
\hat{X}_\mathit{TM} = U_\mathit{TM} \myconv \hat{S}_\mathit{TM} \myconv V_\mathit{TM}^{\myhtrans}
$ where $\hat{S}_\mathit{TM} = \operatorname{diag}(\lambda_{T, 1}, \cdots,\lambda_{T, r},
\underset{Q-r}{\underbrace{Z_{T}, \cdots, Z_{T}}} )$.
Then,
$\hat{X}_\mathit{TM}$ is a low-rank approximation to $X_\mathit{TM}$ since
the following rank inequality holds.
\begin{equation}
\begin{matrix}
\operatorname{rank}(\hat{X}_\mathit{TM}) \equiv \sum\nolimits_{k=1}^{r} \operatorname{rank}(\lambda_{T, k})
\le \operatorname{rank}(X_\mathit{TM}) \;.
\end{matrix}
\end{equation}

\vspace{0.5em}
Furthermore,
it is not difficult to verify that equation (\ref{equation:generalized-Eckart-theorem}) is equivalent to
the following equation in the form of canonical matrices (i.e., slices of Fourier transformed t-matrices).
\begin{equation}
\begin{aligned}
\tilde{X}_\mathit{TM}^\mathit{approx}(i) = \mathop{\operatorname{argmin}}\nolimits_{Y_\mathit{mat} \in \mathbb{C}^{D_1\times D_2}} \|\tilde{X}_\mathit{TM}(i) - Y_\mathit{mat}\|_F \\
\text{subject to} \operatorname{rank}(Y_\mathit{mat}) \le r \;, \; 1\le i \le I
\end{aligned}
\label{equation:minimization}
\end{equation}
where
$\tilde{X}_\mathit{TM}^\mathit{approx}$ and
$\tilde{X}_\mathit{TM}$
respectively denote the Fourier transform of
${X}_\mathit{TM}^\mathit{approx}$ and
${X}_\mathit{TM}$ in equation (\ref{equation:generalized-Eckart-theorem}), and
$\operatorname{rank}(Y_\mathit{mat})$ is the rank of a complex matrix
$Y_\mathit{mat}$ in $\mathbb{C}^{D_1\times D_2}$.

\vspace{0.5em}
On the other hand,
by applying the Fourier transforms to both sides of
equation (\ref{equation:TSVD-low-rank-approximation111}),
equation (\ref{equation:TSVD-low-rank-approximation111}) is transformed to the following equation in the form of canonical matrices (i.e., slices of Fourier transformed t-matrices).
\begin{equation}
\tilde{X}_\mathit{TM}^\mathit{svd}(i) = \tilde{U}_\mathit{TM}(i) \cdot \tilde{S}^\mathit{approx}_\mathit{TM}(i) \cdot \left(\tilde{V}_\mathit{TM}(i) \right)^{H},\;
1 \le i \le I
\label{equation:TSVD-transformed}
\end{equation}
where
$\tilde{X}_\mathit{TM}^\mathit{svd} $,
$\tilde{U}_\mathit{TM}$,
$\tilde{S}_\mathit{TM}^{approx}$ and
$\tilde{V}_\mathit{TM}^{approx}$ respectively denote
the Fourier transforms of
$\hat{X}_\mathit{TM} $,
$U_\mathit{TM} $,
$\hat{S}_\mathit{TM} $ and
$V_\mathit{TM} $ in
equation (\ref{equation:TSVD-low-rank-approximation111}).

The canonical Eckart-Young-Mirsky theorem guarantees the equivalence of equations (\ref{equation:minimization}) and
(\ref{equation:TSVD-transformed}).

\clearpage

\section*{Appendix II}

This appendix contains
an equivalent definition of the Fourier transform and its inverse transform of a multi-way array, and
numerical examples to illustrate some of the definitions in the article.

\vspace{0.5em}
1. Given a multi-way array $X \in \mathbb{C}^{I_1 \times \cdots \times I_N}$, the {order-$N$} Fourier
transform of $X$ is given by the following multi-mode tensor multiplication, which is equivalent to the Fourier transform definition given in Section \ref{FourierTransformOfAT-Scalar}.
\begin{equation}
\tilde{X} \doteq F(X) \doteq X \times_1 W_\mathit{mat}^{(I_1)} \cdots
\times_k W_\mathit{mat}^{(I_k)}  \cdots
\times_N W_\mathit{mat}^{(I_N)}
\nonumber
\end{equation}
where $W_\mathit{mat}^{(I_k)}$ is the Fourier matrix of size $(I_k\times I_k)$.
The $(k_1, k_2)$-th entry of $W_\mathit{mat}^{(I_k)}$ is defined as the following complex number
\begin{equation}
\begin{aligned}
\left(W_\mathit{mat}^{(I_k)}\right)_{k_1, k_2} \doteq
e^{{2\pi \sqrt{-1}\cdot (k_1 - 1)\cdot(k_2 - 1)}\cdot {I_k^{-1}}  }
\end{aligned}
\nonumber
\end{equation}
for all $1\le k_1, k_2 \le I_k$.

\vspace{0.5em}
The inverse Fourier transform is also defined using multi-mode tensor multiplication
\[
X \doteq \tilde{X} \times_1 \left(W_\mathit{mat}^{(I_1)} \right)^{-1} \cdots
\times_k \left( W_\mathit{mat}^{(I_k)} \right)^{-1}  \cdots
\times_N \left(W_\mathit{mat}^{(I_N)} \right)^{-1}
\]
where $\left( W_\mathit{mat}^{(I_k)} \right)^{-1}$ is the the inverse matrix of $W_\mathit{mat}^{(I_k)}$
for all $1\le k \le N$.

\vspace{0.5em}
2. A diagram of the multiplication of two t-scalars, either in the spacial domain or the Fourier domain, is given
in Figure \ref{fig:TscalarMultiplication}.

\vspace{0.5em}
3. An illustrative example of t-matrix multiplication
is given in Figure \ref{fig:TmatrixMultiplication}.
The t-matrices are $C_\mathit{TM} = A_\mathit{TM} \myconv B_\mathit{TM} \in C^{2\times 1} \equiv \mathbb{C}^{3\times 3\times 2\times 1}
$, where $A_\mathit{TM}\in C^{2\times 2} \equiv \mathbb{C}^{3\times 3\times 2\times 2}$
and $B_\mathit{TM}\in C^{2\times 1} \equiv \mathbb{C}^{3\times 3\times 2\times 1}$.

\vspace{0.5em}
4. An example of a generalized tensor (g-tensor) $X_\mathit{GT} \in C^{2\times 3\times 2} \equiv \mathbb{C}^{3\times 3\times 2\times 3\times 2}$ (the size of t-scalars is $3\times 3$) and its generalized mode-$k$ flattening are given in Figure \ref{fig:anExampleGT}.

In the Figure \ref{fig:anExampleGT}, the notation
$(X_\mathit{GT})_{:,:, 1} \in C^{2\times 3}$ denotes the first frame of $X_\mathit{GT}$ and
the notation
$(X_\mathit{GT})_{:,:, 2} \in C^{2\times 3} $ denotes the second frame of $X_\mathit{GT}$.
The notations
$X_\mathit{GT(\mathrm{1})} \in C^{2\times 6}$,
$X_\mathit{GT(\mathrm{2})} \in C^{3\times 4} $,
$X_\mathit{GT(\mathrm{3})} \in C^{2\times 6} $ respectively denote the generalized
mode-$1$,
mode-$2$ and
mode-$3$ flattening t-matrix forms of $X_\mathit{GT}$.

\vspace{0.5em}
5. An example result of generalized mode-$2$ multiplication $M_\mathit{GT} \doteq X_\mathit{GT}
~\myconv_{2}~ Y_\mathit{TM} \in C^{2\times 2\times 2} \equiv \mathbb{C}^{3\times 3\times 2\times 2\times 2}$,
where $X_\mathit{GT} \in C^{2\times 3\times 2}
\equiv \mathbb{C}^{3\times 3\times 2\times 3\times 2}
$
and $Y_\mathit{TM} \in C^{2\times 3} \equiv \mathbb{C}^{3\times 3\times 2\times 3} $ is shown in Figure \ref{fig:TensorMultiplication}.

\begin{figure*}[tbh]
\centering
\includegraphics[width=0.9\textwidth]{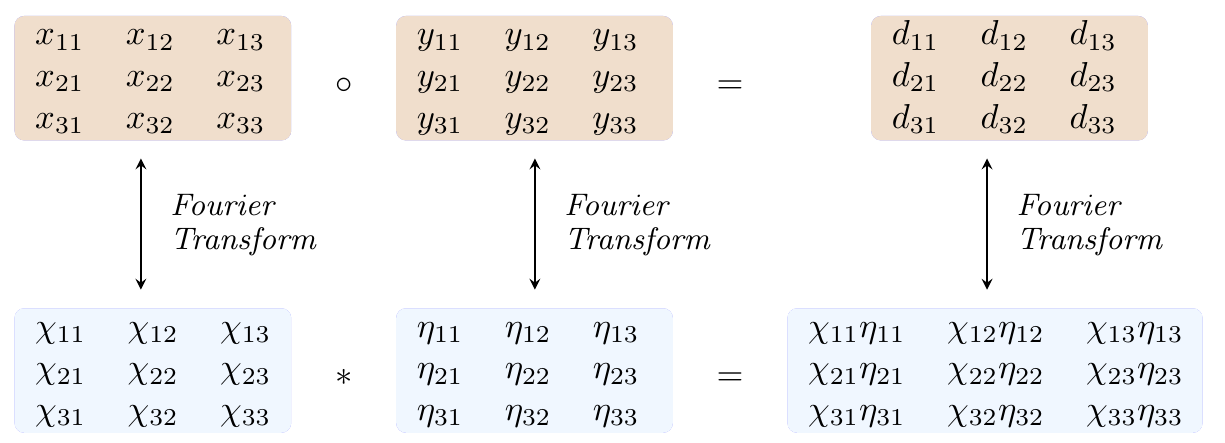}
\caption{A diagram of t-scalar multiplication where the size of t-scalars is $3\times 3$.}
\label{fig:TscalarMultiplication}
\end{figure*}

\begin{figure*}[tbh]
\centering
\includegraphics[width=0.9\textwidth]{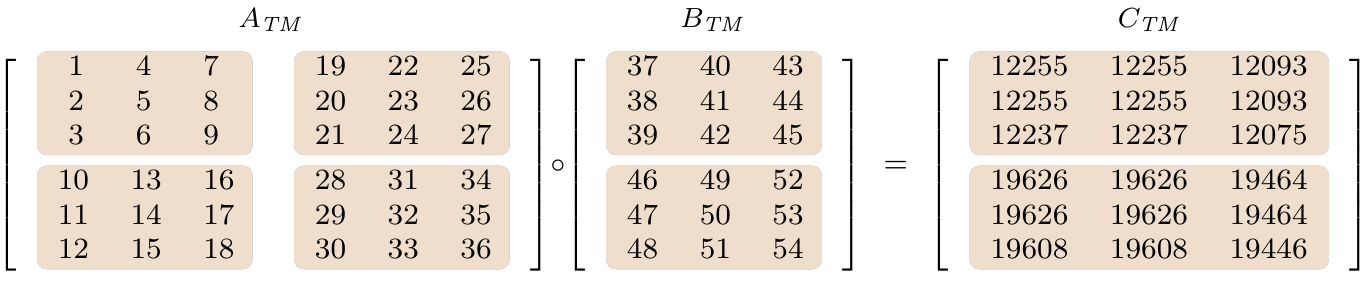}
\caption{An illustrative example of t-matrix multiplication where the size of t-scalars is $3\times 3$.}
\label{fig:TmatrixMultiplication}
\end{figure*}

\begin{figure*}[tbh]
\centering
\includegraphics[width=0.9\textwidth]{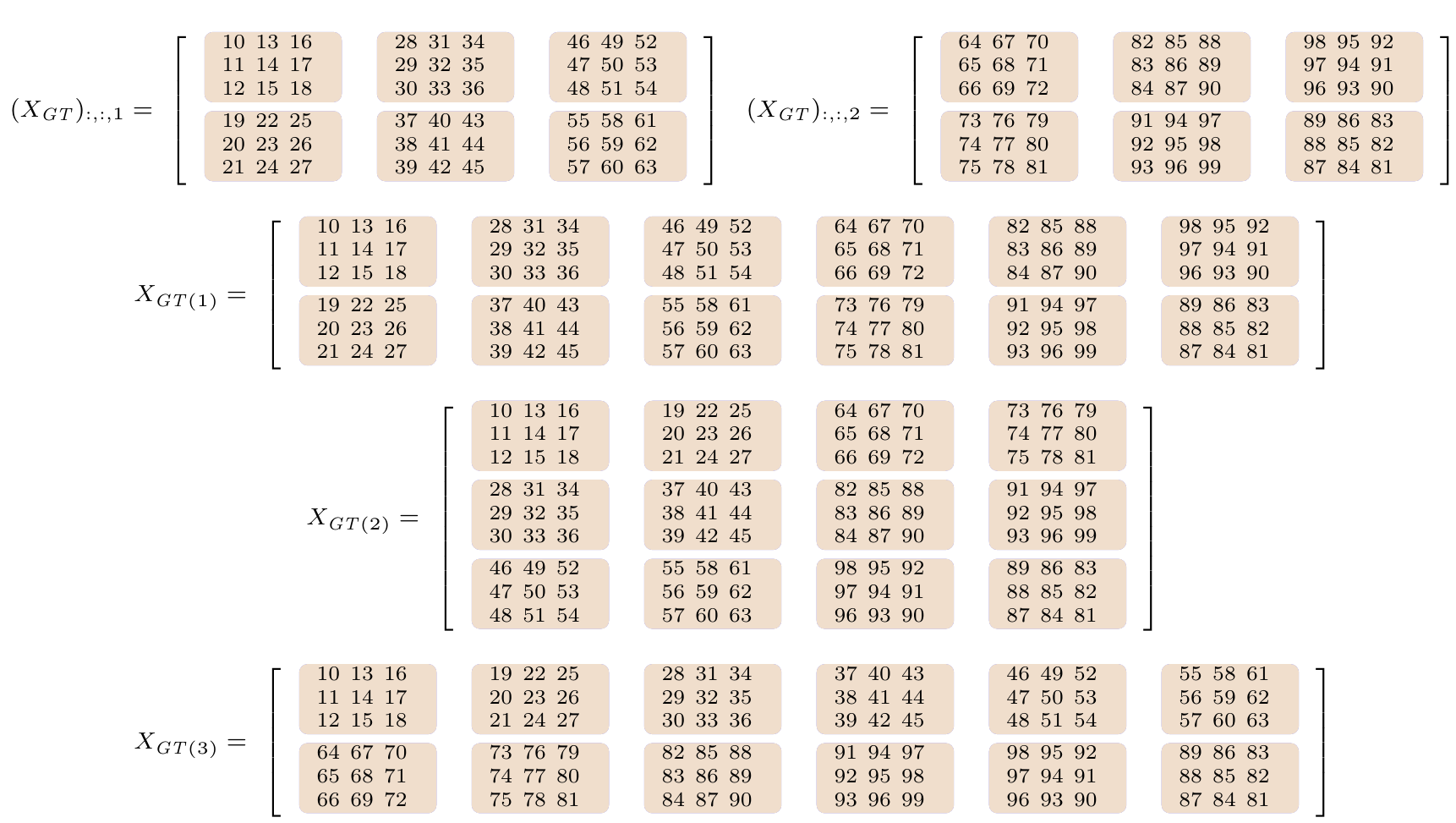}
\caption{An illustrative example of a generalized tensor in $C^{2\times 3\times 2}$ and the mode-$k$ ($k =1, 2, 3$) flattened form of the generalized tensor. }
\label{fig:anExampleGT}
\end{figure*}

\begin{figure*}[tbh]
\centering
\includegraphics[width=0.65\textwidth]{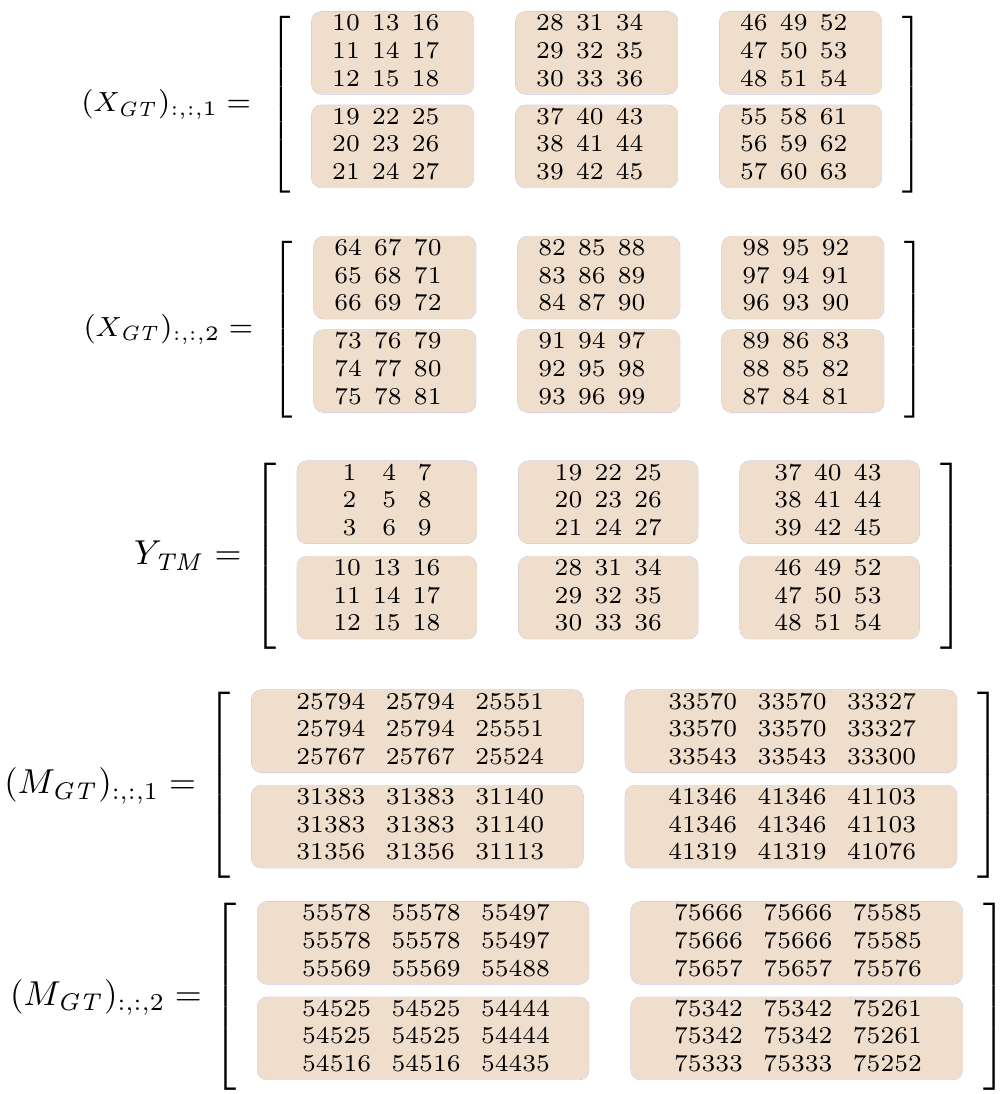}
\caption{An illustrative example of mode-$2$ generalized tensor multiplication
$M_\mathit{GT} = X_\mathit{GT} \circ_{2} Y_\mathit{TM} \in C^{2\times 2\times 2} \equiv \mathbb{C}^{3\times 3\times 2\times 2\times 2}$
where
$X_\mathit{GT} \in C^{2\times 3\times 2} \equiv \mathbb{C}^{3\times 3\times 2\times 3\times 2} $,
$Y_\mathit{TM} \in C^{2\times 3} \equiv \mathbb{C}^{3\times 3\times 2\times 3}$.}
\label{fig:TensorMultiplication}
\end{figure*}

\end{document}